\documentclass[10]{article} % For LaTeX2e
\usepackage[preprint]{tmlr-style-file-main/tmlr}
\usepackage{amsmath,amsfonts,bm}

\def\eqref#1{equation~\ref{#1}}
\def\1{\bm{1}}

\DeclareMathAlphabet{\mathsfit}{\encodingdefault}{\sfdefault}{m}{sl}
\SetMathAlphabet{\mathsfit}{bold}{\encodingdefault}{\sfdefault}{bx}{n}

\usepackage[pdftex]{graphicx}
\usepackage{hyperref}
\usepackage{url}
\usepackage{amsthm}
\newtheorem{theorem}{Theorem}
\newtheorem{lemma}{Lemma}
\newtheorem{remark}{Remark}

\title{On Connecting Deep Trigonometric Networks with Deep Gaussian Processes: Covariance, Expressivity, and Neural Tangent Kernel}
\newenvironment{psmallmatrix}
  {\left(\begin{smallmatrix}}
  {\end{smallmatrix}\right)}

\DeclareMathOperator{\EX}{\mathbb{E}}

\author{\name Chi-Ken Lu \email CL1178@rutgers.edu \\
      \addr Department of Mathematics and Computer Science\\
      Rutgers University Newark
      \AND
      \name Patrick Shafto \email patrick.shafto@rutgers.edu \\
      \addr Department of Mathematics and Computer Science\\
      Rutgers University Newark \\
      School of Mathematics, Institute for Advanced Study \\ 
      Princeton, New Jersey, USA
      }
\def\month{MM}  % Insert correct month for camera-ready version
\def\year{YYYY} % Insert correct year for camera-ready version
\def\openreview{\url{https://openreview.net/forum?id=XXXX}} % Insert correct link to OpenReview for camera-ready version

\begin{document}

\maketitle

\begin{abstract}
Deep Gaussian Process (DGP) as a model prior in Bayesian learning intuitively exploits the expressive power in function composition. DGPs also offer diverse modeling capabilities, but inference is challenging because marginalization in latent function space is not tractable. With Bochner's theorem, DGP with squared exponential kernel can be viewed as a deep trigonometric network consisting of the random feature layers, sine and cosine activation units, and random weight layers. In the wide limit with a bottleneck, we show that the weight space view yields the same effective covariance functions which were obtained previously in function space. Also, varying the prior distributions over network parameters is equivalent to employing different kernels. As such, DGPs can be translated into the deep bottlenecked trig networks, with which the exact maximum a posteriori estimation can be obtained. Interestingly, the network representation enables the study of DGP's neural tangent kernel, which may also reveal the mean of the intractable predictive distribution. Statistically, unlike the shallow networks, deep networks of finite width have covariance deviating from the limiting kernel, and the inner and outer widths may play different roles in feature learning. Numerical simulations are present to support our findings. 

\end{abstract}

%\begin{keywords}%
 %Deep Gaussian process, deep neural network, Bayesian learning, kernel method, kernel composition, deep kernel estimator.
%\end{keywords}

\section{Introduction}
Nearly a decade has passed since the proposal of Deep Gaussian Process (DGP)~\citep{damianou2013deep} which, along with principled uncertainty estimation inherited from Gaussian Process (GP)~\citep{rasmussen2006gaussian}, aimed to exploit the compositional structure like Deep Neural Network (DNN) for superior expressivity and feature learning. Unfortunately, adopting DGP in application remains difficult due to costly computation and challenging optimization~\citep{dutordoir2021deep}.
In the Bayesian setting, computation of exact posterior is impossible because one must marginalize multiple latent functions within the hierarchy. Numerous approximate Bayesian inference schemes, see e.g.~\citep{bui2016deep,salimbeni2017doubly,ustyuzhaninov2020compositional}, have been proposed. %However, a question which should have been asked in the first place: is DGP really that powerful and, if somehow the exact inference were available, one could simultaneously get uncertainty and expressive power without compromising one another? This question might not have an answer yet, but why not {\it partially} tackle it from a maximum a posteriori (MAP) point of view?
Because of the intractability of inference, seemingly basic questions, e.g. the expressivity of DGP, remain unanswered. 
Analytic methods, even only for maximum a posteriori (MAP), would allow further insights. 

One particular approximate DGP inference stands out among others as it does not rely on imposing inducing points on latent functions and makes strong connection with DNN. \cite{cutajar2017random} utilized the concept of expanding the squared exponential (SE) kernels in terms of Gaussian random features and sine/cosine activation~\citep{rahimi2008random}, which allows one to translate a GP with SE kernel into a shallow but infinitely wide trigonometric network. Then, DGP as a cascade of GPs is a random deep bottlenecked network~\citep{agrawal2020wide}, i.e. the activation layers have infinite units but latent output layers are of finite dimension. The bottlenecks ensure the heavy-tailed statistics~\citep{pleiss2021limitations} pertaining to DGPs~~\citep{duvenaud2014avoiding,lu2020interpretable}, unlike the DNNs without bottlenecks are converged into GP~\citep{lee2018deep,matthews2018gaussian}. To pursue MAP of DGP in this context, we shall show that varying prior over the weight parameters translates to different kernel compositions for DGPs~\citep{lu2020interpretable}. Thus, we can apply gradient descent to the squared loss minus the log of prior over weights for obtaining a MAP estimate. More interestingly, the MAP solutions shall be closely related to those obtained from the neural tangent kernel (NTK) regression~\citep{jacot2018neural,arora2019exact}.

Therefore, the deep bottlenecked networks position us to understand the true expressive power of DGPs, whether simply stacking GPs is better than the tricks of kernel composition~\citep{duvenaud2013structure,wilson2016deep,sun2018differentiable} and activation design~\citep{pearce2020expressive}. Nevertheless, DGPs offer appealing flexibility such as multi-fidelity modeling~\citep{kennedy2000predicting,cutajar2019deep,lu2021conditional} and can be regarded as a Bayesian deep kernel learning~\citep{wilson2016deep,ober2021promises,lu2021empirical}. Another critical issue is the general lack of feature learning for kernel based models like GP and DGP. Kernel functions are fixed, not depending on training data whereas the features learned in DNNs are result of back propagating the training error. We shall analyze the finite-width kernels of the random deep bottlenecked networks, the results of which suggest that the learning with a finite-width Bayesian deep network is similar with GP learning but with random kernels~\citep{benton2019function}.

In this paper, we pursue analytical results and investigate the two-layer wide bottlenecked trigonometric network, a proxy of two-layer DGP with SE kernels, and make four main contributions. (i) Covariance: we show the equivalence between the two models as the bottlenecked random networks in the wide limit yield the same exact covariance~\citep{lu2020interpretable}. 
%In weight space view, the derivation of covariance relies on the spectrum of a unique class of random matrices. 
%We also study the multivariate characteristic function for understanding its heavy statistical tails. 
(ii) Expressivity: we show shallow trig networks can approximate a GP with spectra mixture kernel~\citep{wilson2013gaussian} if the features are samples from mixture of Gaussians. 
%A deep kernel learning model~\citep{wilson2016deep} is connected with a deep trig network containing a correlated weight layer. 
In addition, marginal prior distribution~\citep{yaida2020non,zavatone2021exact} of a shallow trig net can be non-Gaussian if an embedding phase shift network is incorporated. (iii) NTK: translating DGPs to the deep trigonometric network representation allows us to derive a closed form NTK for the corresponding DGPs. The expectation is that kernel regression using NTK shall correspond to the exact MAP solution of DGP.
%Exact DGP inference is impossible due to intractable marginalization in function space, and the NTK can shed light on the predictive mean from weight space view. 
(iv) Finite-width effects: We define a kernel estimator for a finite network by marginalizing the random weight parameters. The kernel estimator is then a function of the random features. Mean of the estimator only coincides with the exact DGP kernel in the wide limit, which signifies the difference with the shallow network~\citep{yu2016orthogonal}.
%We can thus approximately calculate the deviation between the mean of estimator and the limiting kernel, and the variance. By connecting DGP with the deep trigonometric networks from covariance perspective, the latter representation can be more flexible, expressive, and transparent than the former through our analytical results including the NTK.     
%\pat{Want a last sentence summarizing how these results fit together to inform our understanding.}

%\pat{the rest is extra, some of which should be integrated above}

The paper has the following organizations. A background for the trigonometric networks, deep Gaussian processes, and the random feature expansion of kernels is introduced in Sec.~\ref{sec:background}. In Sec.~\ref{sec:shallow}, covariance of shallow trig networks with different parameter distributions and its non-Gaussian function distribution are discussed. The derivation of effective kernels for deep trigonometric networks with various parameter distributions is given in Sec.~\ref{sec:deep}. 
%The heavy-tailed property is discussed in Sec.~\ref{sec:character}.
As the connection between deep trigonometric network and DGP is built, Sec.~\ref{sec:ntk} devotes to the derivation of neural tangent kernel. Considering the reality for neural networks, Sec.~\ref{sec:finite} formulates the framework for calculating the correction to covariance as a result of the finite width. Numerical simulations are presented in Sec.~\ref{sec:simulation}. The context of literature in which the present work should be placed can be found in Sec.~\ref{sec:related}, and a conclusion in Sec.~\ref{sec:dis} is provided. 

\section{Background}\label{sec:background}

Consider a parametric function $f_{\bf W}({\bf x})$ which maps input ${\bf x}\in\mathbb R^D$ to real number. In Bayesian settings, given the data $\{{\bf x}_i,y_i\}_{i=1:N}$ denoted by $\mathcal D$, the goal is to obtain the predictive distribution,  
\begin{equation*}
	p(y_*|{\bf x}_*,\mathcal D) = \int d{\bf W}\ p(y_*|f_{\bf W}({\bf x}_*))p({\bf W}|\mathcal D)\:,
\end{equation*} for an unseen input ${\bf x}_*$. A simple likelihood is Gaussian density, $p(y|f_{\bf W}({\bf x}))=\mathcal N(y|f_{\bf W}({\bf x}),\sigma_n^2)$. The posterior is obtained through Bayes rule, $p({\bf W}|\mathcal D)\propto p({\bf W})\prod p(y_i|f_{\bf W}({\bf x}_i))$, with the normalization constant being the evidence or marginal likelihood. In most cases for Bayesian deep neural networks, the marginalization over the parameters ${\bf W}$ is not tractable, and one may seek the maximum a posteriori (MAP) solution. Namely, $p(y_*|f_{\bf\overline W}({\bf x}_*))$ becomes the predictive solution with
\begin{equation*}
	{\bf\overline W}={\rm argmin}\big[-\log p({\bf W})-\sum\log p(y_i|f_{\bf W}({\bf x}_i))\big]\:.
\end{equation*}
%The translation between prior density in weight space and function space is achieved via the integral
%\begin{equation*}
%	p(f)=\int d{\bf W}\ \delta^F(f-f_{\bf W})p({\bf W})\:,
%\end{equation*} where the functional delta function $\delta^F$ is nonzero only when $f=f_{\bf W}$ holds for entire input space. Thus, considering the dimensionality, it is difficult to obtain a density over function directly from the density over weights. 
To understand the translation between the weight and function representations, we shall analytically investigate i) the marginal function prior%, $p(a)=\int d{\bf W}\delta[f_{\bf W}({\bf x})-a]p({\bf W})$, 
for a single input %${\bf x}$ and output $a\in\mathbb R$, 
and ii) the covariance $\int d{\bf W} f_{\bf W}({\bf x})f_{\bf W}({\bf y})p({\bf W})$ in weight representation for comparing with the covariance obtained in function representation.
%Given the functional form of deep neural network $f_{\bf W}({\bf x})$, 

As the basis for theoretical findings in this paper, we outline three prior theoretical results: marginal prior distribution for deep linear neural network~~\citep{zavatone2021exact}, exact covariance of two-layer DGP with squared exponential kernel~\citep{lu2020interpretable}, and the random feature expansion of squared exponential kernel~\citep{rahimi2008random}.  

%\subsection{The Bayesian setting and MAP}
%Let us focus on regression tasks and the data  are given. The ultimate goal in Bayesian learning is to obtain the predictive distribution $p(y_*|{\bf x}_*,\mathcal D)$ through the posterior distribution $p(\mathcal H|\mathcal D)$ over the hypothesis space $\mathcal H$. Modelers may restrict $\mathcal H$ to the space of continuous function, and use a distribution $p(f)$ as their pior.
%We can define the density in function space, 
%\begin{equation}
%	p(f) = \int\ d{\bf W} \delta(f-f_{\bf W})p({\bf W})\:,
%\end{equation}

\subsection{Random neural networks}
Neural networks are a class of parametric models in which one can regard the function output as the outcome of propagating the input through a computational graph consisting of multiple layers of linear and nonlinear mappings. For example, a shallow network can have the following form, 
%n regression problems, a model function $f$ is to serve as a representation for the data $\{{\bf x}_{1:N}\in\mathbb R^d,y_{1:N}\in\mathbb R\}$, i.e. $y=f({\bf x})+\epsilon$ with $\epsilon$ being some irreducible noise. For parametric models such as deep neural network, the function $f({\bf x|w})$ depends on a set of parameters ${\bf w}$ as well as its computational structure (fully connected, convolutional, etc.) and activation (linear, ReLu, tanh, etc.). 
%the optimal values of which are determined by minimizing the loss function, e.g. the mean squared error $\EX_{{\bf x},y\sim P_{\rm data}}[f({\bf x})-y]^2$  
%In probabilistic approach, a prior distribution $p({\bf w}|\theta)$ with hyper-parameters denoted by $\theta$ induces the inductive bias via the joint density $p(f({\bf x}),{\bf w}|\theta)$. The goal in Bayesian inference is to learn the posterior over ${\bf w}$ from data using the Bayes rule,  
\begin{equation}
	%p({\bf w}|{\bf X, y},\theta)=\frac{p({\bf y|w,X})p({\bf w}|\theta)}{p({\bf y|X},\theta)}\:, 
	f({\bf x})={\bf w}\Phi(\Omega{\bf x})\:,\label{shallow_net}
\end{equation} where the input ${\bf x}\in\mathbb R^D$ is sequentially propagated through the feature layer (producing preactivation from multiplying $\Omega\in\mathbb R^{n\times d}$ with input), activation units (element-wise nonlinear mapping $\Phi(\cdot)$), and weight layer (multiplying ${\bf w}\in\mathbb R^{1\times n}$). A deep neural network has similar structure. For instance, %the two-layer model reads,
\begin{equation}
	f({\bf x})={\bf w}_2\Phi\big(\Omega_2{\bf W}_1\Phi(\Omega_1{\bf x})\big)\:,\label{deep_net}
\end{equation} where the matrices in feature layers have $\Omega_1\in\mathbb R^{n_1\times D}$ and $\Omega_2\in\mathbb R^{n_2\times H}$, and in weight layers ${\bf W}_1\in\mathbb R^{H\times n_1}$ and ${\bf w}_2\in\mathbb R^{1\times n_2}$. The integer $H$ represents the width of latent layer output in deep networks. 
%Given the training data, the gradient descent is employed to obtain the optimal parameters. 
%where the implicit dependence of posterior on the functional form of $f$ has been suppressed. The posterior is often intractable and an approximate posterior $q({\bf w})$ is sought for instead~\citep{blundell2015weight,hernandez2015probabilistic}. In the prediction stage, one can obtain the distribution $p(f({\bf x}_*))=\EX_{{\bf w}\sim q}[p(f({\bf x}_*)|{\bf w})]$, which implicitly depends on the hyper-parameters $\theta$ in prior distribution.  

The inductive bias associated with neural networks is connected to the prior distributions from which the random parameters in the computational graph are sampled. How well a model can generalize in Bayesian learning is critically related to its inductive bias~\citep{wilson2020bayesian}. 
While it is usually difficult to describe the inductive bias of neural networks quantitatively, some special cases do permit analytic investigation. \cite{zavatone2021exact} analytically investigated the marginal distribution over the output of deep linear and ReLu networks. The following remark is about a particular shallow linear network.
 \begin{remark}
 Consider the linear network $f({\bf x})={\bf W}{\Omega}{\bf x}$, a special case of Eq.~(\ref{shallow_net}) with $\Phi$ being identity mapping, and the entries in the random matrices ${\Omega}\in\mathbb R^{2\times D}$ and ${\bf W}\in\mathbb R^{1\times2}$ are independent normal, i.e. $\Omega_{ij}\sim\mathcal N(0,\sigma_1^2)$ and ${\bf W}_{ij}\sim\mathcal N(0,\sigma_2^2)$. Then, the marginal distribution over the output is
%A simple case of ${\bf h}\in\mathbb R^2$ can be shown to 
a Laplace distribution $p(f({\bf x}))=\exp(-|f({\bf x})|/\kappa)/2\kappa$ with $\kappa:=\sigma_1\sigma_2|{\bf x}|$. The heavy-tailed character is consistent with the findings in~\citep{vladimirova2019understanding}.
%where the input ${\bf x}\in\mathbb R^d$ and the weight matrices . 
%\cite{zavatone2021exact} demonstrated a trick in deriving the exact marginal distribution over $f$. With the iid entries $[{\bf w}_2]_i\sim\mathcal N(0,\sigma_2^2)$, the conditional distribution $p(f|{\bf W}_1,{\bf x})=\mathcal N(0,\sigma_2^2|{\bf h}|^2)$ where the vector ${\bf h}:={\bf W}_1{\bf x}$. Again, the iid entries $h_i\sim\mathcal N(0,\sigma_1^2|{\bf x}|^2)$ given the iid normal weight matrix entries $[{\bf W}_1]_{ij}\sim\mathcal N(0,\sigma_1^2)$. It turns out that the exact marginal distribution can be computed via,
\end{remark}
\begin{proof} It is easy to observe that the latent output ${\bf h}=\Omega{\bf x}$ has independent components $h_i\sim\mathcal N(0,\sigma_1^2|{\bf x}|^2)$. Similarly, conditional on ${\bf h}$, the output has $f|{\bf h}\sim\mathcal N(0,\sigma_2^2|{\bf h}|^2)$. To obtain the marginal distribution $p(f):=\EX_{{\bf W},\Omega}[p(f|{\bf W},\Omega)]=\EX_{\bf h}[p(f|{\bf h})]$, one can integrate out ${\bf h}$ during the Fourier transformation and then apply the inverse transform~\citep{zavatone2021exact}. Namely, in this particular case with hidden dimension $n={\rm dim}({\bf h})=2$, we can get,
\begin{equation}
	p(f)=\int \frac{dq}{2\pi}d{\bf h}\ e^{iqf} \tilde p(q|{\bf h})p({\bf h})
	=\int \frac{dq}{2\pi}\frac{e^{iqf}}
	{1+\sigma_1^2\sigma_2^2|{\bf x}|^2q^2}
	=\frac{e^{-|f|/\kappa}}{2\kappa}
	%[\int dhe^{-\frac{1+q^2\sigma_1^2\sigma_2^2}{2\sigma_1^2}h^2}]^{n_1}
	\:.
\end{equation} In deriving above, we have used the fact that the Fourier transformation of $p(f|{\bf h})$ is $\tilde p(q|{\bf h})=\exp(-\frac{1}{2}q^2\sigma_2^2|{\bf h}|^2)$ and the residue theorem is applied to complete the last equality.
%Then, use the characteristic function $\tilde p(q|{\bf h})=\exp[-\frac{1}{2}q^2\sigma_2^2|{\bf h}|^2]$ of the conditional distribution $p(f|{\bf h})$ and the prior distribution $p({\bf h})=\prod_i\mathcal N(h_i|0,\sigma_1^2|{\bf x}|^2)$, one can exactly compute the marginal prior distribution $p(f)$. 
\end{proof}
%The marginal prior distribution over single output of a random function can be computed as the following~\citep{zavatone2021exact},
%\begin{equation*}
%	p(f({\bf x}))=\int d\Theta\ p(f({\bf x})|\Theta)p(\Theta)\:,
%\end{equation*} where $\Theta$ stand for the collection of all random parameters within a neural network. 

As the outputs of neural network are not independent given the shared parameters, another perspective of studying the inductive bias is to investigate the distribution over the function values, i.e. $p(f({\bf x}_1),f({\bf x}_2),\cdots,f({\bf x}_N))$, indexed by the set of inputs. This is a more challenging task than the above marginal distribution over the function at single input. Fortunately, the central limit theorem applies when the number of activation units becomes infinity, the multivariate distribution converges to Gaussian, and the limiting statistics only depends on the mean $\EX[f({\bf x})]$ and covariance $\EX[f({\bf x})f({\bf y})]$. Closed form covariance functions can be derived for shallow networks with sigmoidal and ReLu activations~\citep{williams1997computing,cho2009kernel}, but the same techniques do not seem to carry to the deeper networks. As for the deep networks of finite width, various techniques from statistical physics~\citep{dyer2019asymptotics,yaida2020non,roberts2021principles} have been employed to compute the corrections.  
%\begin{align}
%	\EX[f({\bf x}]&=\\
%	\EX[f({\bf x})f({\bf y})]&=
%\end{align}

\subsection{Gaussian process and deep Gaussian process}

In parallel, Gaussian Processes~\citep{rasmussen2006gaussian} (GPs) directly model
%are nonparametric models which represent the inductive bias with a multivariate normal distribution over 
the set of function values with a Gaussian, $p(f({\bf x}_{1:N})|\theta)=\mathcal N(\mu({\bf x}_{1:N}),\Sigma({\bf x}_{1:N},{\bf x}_{1:N}))$, with $\theta$ being the hyper-parameters in the mean function $\mu$ and covariance matrix $\Sigma$, which fully specify the model. 
%The matrix elements in $\Sigma$ come from a predetermined covariance function, $\Sigma_{ij}=k({\bf x}_i,{\bf x}_j)$, which encodes, for instance, the smoothness of function. 
Being Gaussian allows analytic marginalization, which leads to the defining property of the mean function $\EX[f({\bf x}_i)]=\mu({\bf x}_i)$ and the covariance function, 
\begin{equation*}
	\Sigma_{ij}=\EX\{[f({\bf x}_i)-\mu({\bf x}_i)][f({\bf x}_j)-\mu({\bf x}_j)]\}=k({\bf x}_i,{\bf x}_j)\:,
\end{equation*} where $k$ is a predetermined kernel function, e.g. squared exponential function. In addition, a closed form for the marginal likelihood $p({\bf y|X},\theta)=\EX_{f\sim\mathcal N(\mu,\Sigma)}[p({\bf y}|f({\bf X}))]$ can be obtained if a Gaussian likelihood is adopted, with which the optimal hyper-parameters is determined. Conditional on the prior observations, the responses ${\bf y}_*$ at a set of inputs ${\bf X}_*$ then follows another normal distribution $\mathcal N({\bf y}_*|\mu_*,\Sigma_*)$ with conditional mean,
\begin{equation}
	\mu_*= \Sigma({\bf X}_*,{\bf X})[\Sigma({\bf X})+\sigma_s^2I]^{-1}{\bf y}\:,\label{cond_mean}
\end{equation} and conditional covariance,
\begin{equation}
	\Sigma_* = \Sigma({\bf X}_*)-\Sigma({\bf X}_*,{\bf X})[\Sigma({\bf X})+\sigma_s^2I]^{-1}\Sigma({\bf X},{\bf X}_*)\:,\label{cond_cov}
\end{equation} where we take the prior mean to be zero, $\mu=0$, for easing the notation, and hyper-parameter $\sigma_s^2$ denoting the noise variance connecting $f$ to the observations.  
%can be described as the marginal prior,
%\begin{equation}
%	p(f|\theta)=\int d{\bf w}\ p(f|{\bf w})p({\bf w}|\theta)\:,
%\end{equation} where the prior $p({\bf w}|\theta)$ for the parameters is determined by a set of hyper-parameters $\theta$. 
%the network parameters are random variables, and the estimation of posterior distribution $q({\bf w})$ is done via optimizing the 

%In the followings, we briefly review the random feature expansion which allows connect shallow trig networks with GPs carrying  Gaussian kernel. Moreover, a covariance perspective of deep Gaussian process is introduced.

Among many extensions of GPs for enhancing expressivity, e.g. warped GP in~\citep{snelson2004warped}, Deep Gaussian Processes (DGPs)~\citep{damianou2013deep} are a general hierarchical composition of GPs. The compositional structure enhances its expressive power, e.g. a GP with SE kernel can not fit a step function well but a DGP can. Consider for simplicity a two-layer function $f({\bf x})=f_2({\bf f}_1({\bf x}))$ where the input ${\bf x}\in\mathbb R^D$ is mapped to the hidden output ${\bf h=f}_1({\bf x})\in\mathbb R^H$ and then to a real output $f_2({\bf h})$. The hidden layer with finite $H$ is referred to as the bottleneck in~\citep{agrawal2020wide,aitchison2020bigger}. DGP is defined by the joint density $p(f_2({\bf f}_1({\bf x}_{1:N}))$,
\begin{equation*}
	\mathcal N(f_2({\bf h}_{1:N})|0,\Sigma_2({\bf H}))\prod_{i=1}^H\mathcal N(h_i({\bf x}_{1:N})|0,\Sigma_1({\bf X}))\:,
\end{equation*} where subscripts in covariance matrices remind us that the covariance functions in different layers can be different. The hidden output ${\bf H}$ is a data matrix consisting of vector-valued hidden functions ${\bf h}({\bf x}_{1:N})$, entering as input to second GP. In Bayesian inference, the marginalization of the hidden random variables ${\bf h}$ is not tractable, which leads to various approximation schemes including variational inference~\citep{salimbeni2017doubly,salimbeni2019deep,haibin2019implicit,ustyuzhaninov2020compositional,ober2021global} and expectation propagation~\citep{bui2016deep}. 
%and moment matching incorporated with hyperdata learning~\citep{lu2021empirical}.

%In a similar spirit as the work, the work~\citep{lu2020interpretable}
One advantage of modeling with the function space view, such as GP, is that we can augment the model by imposing constraint on the function through inducing points~\citep{titsias2009variational,titsias2010bayesian}, i.e. the random function has to pass through a set of points, $f({\bf z}_{1:M})=u_{1:M}$, in the absence of noise. Those points can be treated as additional hyper-parameters to be optimized (empirical Bayes), or can be treated as random variable so that one has to infer their distribution in a full Bayes setting. In the context of DGP, these inducing points can serve as hidden function's support in variational inference~\citep{salimbeni2017doubly}, or they can be interpreted as the low fidelity observations in multi-fidelity regression problems~\citep{kennedy2000predicting,cutajar2019deep}. However, it becomes less straightforward to incorporate these inducing points into the deep neural networks from a random weight space view~\citep{ober2021global}.
%, but a recent work~\citep{ober2021global} 

An alternative scheme for inference with DGP models is to view DGP as a GP at the level of the marginal prior, i.e. the hidden function ${\bf f}_1$ being marginalized out from the joint, which is similar to the partially collapsed inference in Gibbs sampling~\citep{park2009partially} and deep kernel learning~\citep{wang2020physics}.  The idea was motivated by the observation that the covariance of the marginal prior distribution over the array of function values taken at inputs ${\bf X}$, 
\begin{equation}
	p({\bf f}|{\bf X})=\int d{\bf F}_1\ p({\bf f}_2|{\bf F}_1)p({\bf F}_1|{\bf X})\:,
\end{equation} can be computed analytically~\citep{lu2020interpretable}. As such, an approximating distribution $q({\bf f|X})=\mathcal N(0,\Sigma_{\rm eff})$ with the matched covariance $[\Sigma_{\rm eff}]_{ij}=\EX[f_2({\bf f}_1({\bf x}_i))f_2({\bf f}_1({\bf x}_j))]$ can be plugged into the standard GP inference pipeline. The compositional hierarchy incorporates all scales from layers into the effective kernels, e.g. $k_{\rm eff}=\sigma_2^2\big\{1+2\frac{\sigma_1^2}{\ell_2^2}[1-\exp(-\frac{d^2({\bf x}_i,{\bf x}_j)}{2\ell_1^2})]\big\}^{-\frac{1}{2}}$ for 2-layer DGP with SE kernels in both zero-mean GPs, and the multi-scale character enables capturing complex patterns in some time series data~\citep{lu2021empirical}. Moreover, the model augmentation incorporating latent function supports as additional hyper-parameters was shown to have better generalization~\citep{lu2021empirical}. The closed form kernel for the 2-layer DGP with learnable latent function support is in the following lemma. The proof can be found in~\citep{lu2021conditional}. 
 
\begin{lemma}
Consider the two-layer DGP, $f({\bf x})=f_2({\bf f}_1({\bf x}))$, where the latent functions, ${\bf f}_1:\mathbb R^D\mapsto\mathbb R^H$ being a vector-valued GP and $f_2:\mathbb R^H\mapsto\mathbb R$ being a GP with SE kernel. The latent function is conditioned on the support, ${\bf f}_1({\bf z}_{1:M})={\bf u}_{1:M}$. The covariance has the following closed form~\citep{lu2021conditional}, 
\begin{equation}
	\EX_{{\bf f}_1}\EX_{f_2|{\bf f}_1}[f({\bf x})f({\bf y})]=\prod_{i=1}^H\frac{e^{-\frac{[\mu_{*,i}({\bf x})-\mu_{*,i}({\bf y})]^2}{2(1+\delta_i^2)}}}{\sqrt{1+\delta^2_i}}\:,\label{DGP_covariance_f}
\end{equation} where the conditional means $\mu_{*,i}({\bf x})$ and $\mu_{*,i}({\bf y})$ are associated with the conditional Gaussian density $p(f_{1,i}({\bf x}),f_{1,i}({\bf y})|{\bf z}_{1:p},u_{i,1:p})$, and the positive value $\delta^2_i=\Sigma_*({\bf x,x})+\Sigma_*({\bf y,y})-2\Sigma_*({\bf x,y})$.   
\end{lemma}   

\subsection{Random feature expansion}

To connect neural networks and above GPs with SE kernel, the following theorem based on the Bochner's theorem is needed. Its proof was provided in~\cite{rahimi2008random}. 
%one can consider the entries in the feature matrix $\Omega$ in Eq.~(\ref{shallow_net}) are fixed and sampled from a normal distribution, the activation $\Phi(~)=\cos(~)$, and the weight parameters are independent normal variables.  

\begin{theorem} 
%\pat{we are stating this without proof because the theorem was proven previously? We could be a bit clearer about who proved it and where.} 
The shallow cosine network~\citep{sopena1999neural,gal2015improving}, 
\begin{equation}
	f({\bf x})=\sqrt{\frac{2}{n}}\sum_{i=1}^n w_i\cos[\omega_i\cdot({\bf x}-{\bf z}_i)+b_i]\:,\label{cosine_network}
\end{equation} is a random parametric function mapping an input ${\bf x}\in\mathbb R^D$ to $\mathbb R$. The collection of independent and normal weight variables, $w_{1:n}\sim\mathcal N(0,\sigma^2)$, and bias $b_{1:n}\sim{\rm Unif}[0,\pi]$. In above expression, ${\bf z}_{1:n}\in\mathbb R^D$ are a set of shift vectors, and are referred to as inducing points in GP literature~\citep{gal2015improving}. The random network has zero mean, and the covariance converges to,
%The general Gaussian kernel $k({\bf x, y})=\exp[-\frac{1}{2}({\bf x-y})^T\Lambda({\bf x-y})]$ is a model of similarity between inputs 
%${\bf x,y}\in\mathbb R^d$ with metric $\Sigma&  is approximated by the following expansion,
\begin{equation}
	%e^{-\frac{\lambda^2|{\bf x-y}|^2}{2}}
	%k({\bf x, y})
	\EX[f({\bf x})f({\bf y})]
	\rightarrow\sigma^2\exp[-\frac{1}{2}({\bf x-y})^T\Lambda({\bf x-y})]\:,\label{network_covariance}
	%\frac{2}{N}\sum_{k=1}^N\cos[\Omega_k({\bf x-z}_k)+b_k]
	%\cos[\Omega_k({\bf y-z}_k)+b_k]\:,
\end{equation} in the limit $n\rightarrow\infty$ if the random vectors $\{\omega_{1:n}\in\mathbb R^D\}$ are samples from a multivariate normal distribution $\mathcal N(0,\Lambda)$. 
%The random bias parameters $b_k$ are samples from a uniform distribution with the range $[0,\pi]$. 
\end{theorem}

%\subsection{Exact output distribution for single input}

\section{Shallow trigonometric network}\label{sec:shallow}
An alternative for the shallow networks in Eq.~(\ref{cosine_network}) which yields the same SE covariance was proposed in~\citep{cutajar2017random}. With the feature vectors $\omega_{1:n}\in\mathbb R^D$, and the random variables $w^c_{1:M}$ and $w^s_{1:M}$ associated with the cosine and sine postactivation, respectively, we can write the random function as, 
%Here, we consider the parametric function with $d$-dimensional input and real output, 
\begin{align}
	f({\bf x}) &= \frac{1}{\sqrt{n}}\sum_{i=1}^{n} w^c_i\cos(\omega_i\cdot{\bf x}) +w^s_i\sin(\omega_i\cdot{\bf x})\:,\label{shallow_trig_network}\\
	& = {\bf w}\Phi(\Omega{\bf x})\:,
\end{align} in which the compact notation in the second line has ${\bf w}=[w^c_1,w^c_2,\cdots,w^c_n,w^s_1,\cdots,w^s_n]\in\mathbb R^{1\times2n}$ and $\Omega=[\omega_1,\cdots,\omega_n]^T\in\mathbb R^{n\times D}$. Activation here is a doublet which reads $\Phi(~)=\begin{psmallmatrix}\cos(~)\\ \sin(~)\end{psmallmatrix}$.
%In addition, the $2M$ post-activations from the cosine and sine units are concatenated in $\Phi\in\mathbb R^{2M}$. %concatenates if the concatenated vector of random variables $\{{\bf w}, {\bf v}\}$ are sampled from independent distribution $N(0,\sigma^2I_{2M})$ .

Based on the same argument in~\cite{rahimi2008random}, Eq.~(\ref{shallow_trig_network}) represents a finite-basis model for random smooth functions whose covariance converges to some fixed form in the limit of large $n$. If the features in $\Omega$ are sampled from a distribution and remain fixed, then one can infer the weight parameters ${\bf w}$ given the data (or hyperdata in~\cite{lu2021empirical}) ${\bf Z,u}$, the prior distribution $p({\bf w})=\mathcal N(0,\sigma^2I_{2n})$, and observation noise variance $\sigma_s^2$. The notation means ${\bf Z}=\begin{psmallmatrix}{\bf z}_1,&\cdots,&{\bf z}_M\end{psmallmatrix}\in\mathbb R^{D\times M}$ and ${\bf u}=\begin{psmallmatrix}u_1,&\cdots,&u_M \end{psmallmatrix}^T\in\mathbb R^{M\times1}$. Following the linear Bayesian learning~\citep{rasmussen2006gaussian}, the posterior reads
\begin{equation}
	p({\bf w}|{\bf Z,u})=\mathcal N({\bf w}|{\bar{\bf w}},A^{-1})\:,
\end{equation} with the conditional mean and precision matrix,
\begin{align}
	{\bar{\bf w}}^T&=\sigma_s^{-2}A^{-1}\Phi(\Omega{\bf Z}){\bf u}\:,\\
	A&= \sigma_s^{-2}\Phi(\Omega{\bf Z})\Phi^T(\Omega{\bf Z}) +\sigma^{-2}I_{2n}\:,
\end{align} where 
%the shape of the data matrix ${\bf Z}$ consisting of $N$ data points is $(N, d)$ and 
the postactivation matrix $\Phi(\Omega{\bf Z})$ has shape $(2n,M)$. Furthermore, the distribution over the predicted value at a new input, $y_*={\bf w}\Phi(\Omega{\bf x}_*)$, is still a Gaussian with mean,
\begin{equation*}
	 \bar f_*=K_*\big(\sigma_s^2I_{2n}+K\big)^{-1}{\bf u}\:,
\end{equation*} and variance
\begin{equation*}
	\sigma^2_*=\sigma_s^2+K_{**}-K_*\big(\sigma_s^2I_{2n}+K\big)^{-1}K_*^T\:,
\end{equation*} where we have used the kernel expression $K_*=\sigma^2\Phi_*^T\Phi$, $K_{**}=\sigma^2\Phi_*^T\Phi_*$ and $K=\sigma^2\Phi^T\Phi$~\citep{rasmussen2006gaussian}. The shorthand notation has $\Phi_*=\Phi(\Omega{\bf x}_*)$ and $\Phi=\Phi(\Omega{\bf Z})$. The above result is thus an approximation for GP regression. 

In the framework of GP regression, one way to enhance the expressive power of the nonparametric model is, ironically, to form a linear combination of different kernels and treat the coefficients as hyper-parameters optimizing the evidence. The classic regression on Mauna Loa dataset in~\cite{rasmussen2006gaussian} adopts the SE kernel along with rational quadratic and periodic kernels. One may also view the spectra mixture kernel~\citep{wilson2013gaussian} as a special kernel composition. For Bayesian neural network, on the other hand, the prior function distribution induced by prior parameter distribution~\citep{wilson2020bayesian,zavatone2021exact} encodes the expressive power. In practice, design of activation in a network was shown to yield good extrapolation~\citep{pearce2020expressive}. In the following two subsections, we shall introduce two ideas improving the expressivity associated with the trig network representation of GP.

\subsection{Features from mixture of Gaussians and spectra mixture kernel}
%Besides the connection between Gaussian kernel $k$ and Gaussian density $p(\Omega)$, the work~\citep{rahimi2008random} pointed out the random expansion of Laplacian and Cauchy kernels with their corresponding single-mode density. However, the Gaussian kernel has a strong inductive bias favoring smooth function, which in some cases, e.g. Mauna Loa data set, fails to capture periodic patterns. The following lemma suggests a way to obtain the Spectra Mixture (SM) kernel~\citep{wilson2013gaussian} by sampling the random frequencies from mixture of Gaussians.   
Following the work of~\citep{rahimi2008random}, one can obtain a shallow trig network representation of GP regression model with SE kernel if the features $\Omega$ are sampled from a normal distribution. Similarly, the GP regression models with Laplacian and Cauchy kernels can have their network representation if the features are sampled from different single-mode distributions. The following lemma show that the model with spectra mixture kernel is equivalent to the shallow trig network if the features are sampled from a mixture of Gaussians.

\begin{lemma} If the features are sampled from a mixture of multivariate Gaussians, $\omega_{1:n}\sim\sum_a\pi_a\mathcal N(\mu_a,\Lambda_a)$ with positive $\pi$'s, and the weight ${\bf w}\sim\mathcal N(0,\sigma^2I_{2n})$, then the covariance of outputs in Eq.~(\ref{shallow_trig_network}) converges to the spectrum mixture kernel, 
\begin{equation}
	k({\bf x}, {\bf y})=\sigma^2\sum_a\pi_a\cos[\mu_a^T({\bf x-y})]e^{-\frac{({\bf x-y})^T\Lambda_a({\bf x-y})}{2}}\:,
\end{equation} in the wide network limit $n\rightarrow\infty$. 
\end{lemma}
\begin{proof}
As the weight parameters are independent, one can easily see that the covariance in the large $n$ limit converges to
\begin{align*}
	\EX[f({\bf x})f({\bf y})]&\rightarrow\sigma^2\int d\omega\ p(\omega)\cos[\omega\cdot({\bf x-y})]\\
	&=\sigma^2{\rm Re}\sum_a\pi_a\int d\omega\ \mathcal N(\omega|\mu_a,\Lambda_a)e^{i\omega\cdot({\bf x-y})}\\
	&=\sigma^2\sum_a\pi_a\cos[\mu_a\cdot({\bf x-y})]\exp[-\frac{1}{2}({\bf x-y})^T\Lambda_a({\bf x-y})]\:.
\end{align*}In the first equality, ${\rm Re}$ refers to as the operation of taking real part.
\end{proof}

%\pat{define Re used in proof?}

\subsection{Prior distribution over the network output}\label{subsec:phase} 
%\cite{dyer2019asymptotics,roberts2021principles}
%\cite{yaida2020non,zavatone2021exact}
Here, we investigate the marginal prior function distribution $p(f)=\int d{\bf w}\ p(f|{\bf w})p({\bf w})$ induced by the prior weight distribution $p({\bf w})$. Following the technique in Remark~1, we can conclude that the prior function distribution associated with the shallow trig network in Eq.~(\ref{shallow_trig_network}) is Gaussian, independent of the feature number $n$.
\begin{remark} The probability density over the function Eq.~(\ref{shallow_trig_network}) for a single input is always a Gaussian with zero mean and variance $\sigma^2$, independent of the width $n$ and of the sampling distribution $p(\Omega)$.
%\begin{proof} One can follow the approach in \citep{zavatone2021exact}
%\end{proof}
\end{remark}
\begin{proof}
Given ${\bf w}$ is independent normal with variance $\sigma^2$, the conditional distribution $p(f|\Phi)$ is also a normal with variance $\frac{\sigma^2}{n}\sum_{i=1}^n\cos^2\omega_1\cdot{\bf x}+\cdots\cos^2\omega_n\cdot{\bf x}+\sin^2\omega_1\cdot{\bf x}+\cdots+\sin^2\omega_n\cdot{\bf x}=\sigma^2$. Thus, $p(f({\bf x}))=\mathcal N(0,\sigma^2)$.
\end{proof}

It was suggested that the superior expressive power of deep linear network and ReLu network is related to their non-Gaussian prior function distribution~\citep{vladimirova2019understanding,roberts2021principles,zavatone2021exact}. Besides the network with finite width which lifts the Gaussianity~\citep{yaida2020non}, the following shallow network $f_{\psi}:\mathbb R^D\mapsto\mathbb R$ with modified preactivation, 
%Now we consider a modified version of the single-layer network,
\begin{equation}
	f_{\psi}({\bf x}) = \frac{1}{\sqrt{n}}\sum_{i=1}^{n} w^c_i\cos[\omega_i\cdot{\bf x}+\psi({\bf x})] +w^s_i\sin[\omega_i\cdot{\bf x}-\psi({\bf x})]\:,\label{single_network_phi}
\end{equation}incorporating a phase shift network $\psi({\bf x})$ is proposed to lift the Gaussianity. 
%and its prior output distribution is non-Gaussian. \pat{the previous is a bit opaque}
%We can employ the method of Gauss-Hermite quadrature~\cite{greenwood1948zeros} to estimate the integral
\begin{lemma} The Fourier transformed $\tilde p(q)$ associated with the prior distribution over the output in Eq.~(\ref{single_network_phi}) is,
\begin{equation}
	\tilde p(q)=e^{-\frac{1}{2}q^2\sigma^2}\int d\omega p(\omega)%\mathcal N(\Omega|0,\lambda^2)
	e^{\frac{1}{2}q^2\sigma^2\sin\psi({\bf x})\sin2\omega\cdot{\bf x}}\:,\label{nonGaussian_prior}
	%&\approx e^{-\frac{\sigma^2q^2}{2}}\sum_i2c_i\cosh\big[\frac{\sigma^2\sin\psi({\bf x})\sin(2\omega_i{\bf x})}{2}q^2\big]
\end{equation} where the feature $\omega\in\mathbb R^D$ are sampled from the high dimensional normal distribution $p(\omega)=\prod_{d=1}^D\mathcal N(\omega_d|0,\sigma_d^2)$.%where the coefficients $c$'s are given in~~\cite{greenwood1948zeros} and the the discrete set $\omega$'s are roots of Hermite polynomials.
\end{lemma}

It can be seen that the phase shift network $\psi({\bf x})$ lifts the Gaussian character of the prior distribution, but the intractable high-dimensional integral in Eq.~(\ref{nonGaussian_prior}) stands in the way of obtaining a closed form for its characteristic function. Nevertheless, one can proceed with the approximation of Gauss-Hermite quardature~\citep{greenwood1948zeros}. Consider the case where the variances $\sigma_{1:D}^2=\sigma_F^2$ associated with the features in all dimensions are the same, and after including the most relevant terms,
\begin{equation}
	\tilde p(q)\approx e^{-\frac{1}{2}q^2\sigma^2}
	(\frac{\lambda_0}{\sqrt{\pi}})^D
	\big\{
	1+2\frac{\lambda_1}{\lambda_0}\sum_{d=1}^D
	\cosh[{\frac{1}{2}q^2\sigma^2\sin\psi({\bf x})\sin(2\sqrt{2}\sigma_Fz_1x_d)}]
	\big\}\:,
\end{equation} where the coefficients $\lambda_0\approx1.181$ and $\lambda_1\approx0.295$ are given in~\citep{greenwood1948zeros} and $z_1\approx1.22$ is the nonzero root of the third order Hermite polynomial. Consequently, the characteristic $\tilde p$ obtains a non-Gaussian correction $\propto q^4e^{-q^2\sigma^2}$ for small Fourier component $q$. 

\section{Deep trigonometric network}\label{sec:deep}

Now we proceed to consider the deep trigonometric network proposed in~\cite{cutajar2017random}. With the same notation as the shallow network, the deep trigonometric network of interest has the following matrix representation,
%The deep trigonometriconometric network is a representation of composite function $f({\bf x})=f_2({\bf f}_1({\bf x}))$,
\begin{equation}
	%{\bf f}_1({\bf x})&=\frac{1}{N_1}\sum_i{\bf w}_i^{(1)}\odot\cos\Omega_i^{(1)}{\bf x}
	%+{\bf v}_i^{(1)}\odot\sin\Omega_i^{(1)}{\bf x}\\
	%f_2({\bf y})&=\frac{1}{N_2}\sum_jw^{(2)}_j\cos\Omega_j^{(2)}{\bf y} +
	%v^{(2)}_j\sin\Omega_j^{(2)}{\bf y}
	f({\bf x})={\bf w}_2\Phi(\Omega_2{\bf W}_1\Phi(\Omega_1{\bf x}))\:,\label{deep_trig_network}
\end{equation} in which the random weight matrices ${\bf w}_2\in\mathbb R^{1\times2n_2}$, ${\bf W}_1\in\mathbb R^{H\times2n_1}$ and the feature matrices $\Omega_2\in\mathbb R^{n_2\times H}$, $\Omega_1\in\mathbb R^{n_1\times D}$. Here, the hidden output ${\bf h}={\bf W}_1\Phi(\Omega_1{\bf x})$ has bottleneck~\citep{agrawal2020wide} dimension $H$ collecting the $n_1$ postactivations. Besides the compositional hierarchy which makes the function more expressive than its shallow counterpart, one can also adopt different prior distribution over the weight and feature matrices. In the following three subsections, we shall discuss the cases of (i) the entries in ${\bf W}_1$, ${\bf w}_2$, $\Omega_1$, and $\Omega_2$ are all independent normal, which corresponds to the zero-mean two-layer DGP with SE kernels~\citep{lu2020interpretable}, (ii) same as in (i) but in the first layer the weight entries in ${\bf W}_1\sim p({\bf w}_1|{\bf Z,U})$ are not independent, which corresponds to the two-layer DGP for multi-fidelity regression~\citep{lu2021conditional} and hyper-data learning~\citep{lu2021empirical} with ${\bf Z, U}$ acting as the support in the latent function, and (iii) same as in (ii) but the feature matrix $\Omega_2$ consists of samples from the mixture of Gaussians, which corresponds to the two-layer DGP with outer GP using the SM kernel.

\subsection{Deep trig net covariance and random matrix spectrum}

To show that the deep trigonometric network yields the same covariance as the two-layer DGP when the entries in weight matrices have independent normal distribution, the spectrum of the following square random matrix with dimension $2n_1$, 
\begin{equation}
	G=[\Phi(\Omega_1{\bf x})-\Phi(\Omega_1{\bf y})][\Phi(\Omega_1{\bf x})-\Phi(\Omega_1{\bf y})]^T\:,
\end{equation} is critical in determining the statistics of network outputs. 
%$G$ is a $2n_1$-dimensional square matrix with entries determined by the random feature vectors $\omega_{1:n_1}$. 

\begin{remark} The square matrix $G$ has $(2n_1-1)$ zero eigenvalues and one nonzero eigenvalue. If the set of feature vectors $\{\omega_{1:n_1}\}$ are sampled from $\mathcal N(\omega|0,I_D)$, then the nonzero eigenvalue shall converge to the following,
\begin{align*}
	\big|\Phi({\bf x})-\Phi({\bf y})\big|^2 &= 
	\frac{1}{n_1}\sum_{i=1}^{n_1}(\cos\omega_i{\bf x}-\cos\omega_i{\bf y})^2+(\sin\omega_i{\bf x}-\sin\omega_i{\bf y})^2\\
	&\rightarrow2-2k_{\rm SE}({\bf x,y})\:,
\end{align*} in the limit $n_1\rightarrow\infty$. $k_{\rm SE}({\bf x,y})=\exp[-\frac{1}{2}|{\bf x-y}|^2]$ stands for the squared exponential covariance function with all hyper-parameters set to unity.
\end{remark}
\begin{proof} First, one can view ${\bf v}=\Phi(\Omega_1{\bf x})-\Phi(\Omega_1{\bf y})$ as a fixed vector in the $2n_1$ dimensional space. The entries read $\frac{1}{\sqrt{n_1}}[\cos\omega_1\cdot{\bf x}-\cos\omega_1\cdot{\bf y},\cdots,\cos\omega_{n_1}\cdot{\bf x}-\cos\omega_{n_1}\cdot{\bf y},\sin\omega_1\cdot{\bf x}-\sin\omega_1\cdot{\bf y},\cdots]$. Then, one can in principle find out the orthogonal subspace, spaned by the set of vectors $\{{\bf v}_{\perp,1:(2n_1-1)}\}$, to ${\bf v}$ in the space. Thus, we have ${\bf v}^T{\bf v}_{\perp}=0$, one can write the zero eigenvalue equations,
\begin{equation*}	
	%G(\Phi_x-\Phi_y)=|\Phi_x-\Phi_y|^2(\Phi_x-\Phi_y) 
	G{\bf v}_{\perp}={\bf v}{\bf v}^T{\bf v}_{\perp} = 0\:,
\end{equation*} and the only nonzero eigenvalue one,
\begin{equation*}
	G{\bf v}=|{\bf v}|^2{\bf v}\:.
\end{equation*} %Now the nonzero $|{\bf v}|^2=$
\end{proof}

With the knowledge of the spectrum of random matrix $G$, now we can go on to derive the desired covariance of deep trigonometric network.%arrive at the following lemma.
\begin{lemma} The covariance of the deep trigonometric network in Eq.~(\ref{deep_trig_network}),
\begin{equation}
	\EX_{{\bf W}_1}\bigg\{
	\EX_{{\bf w}_2|{\bf W}_1}\big[
	f({\bf x})f({\bf y})
	\big] 
	\bigg\} \rightarrow k_{\rm DGP}({\bf x,y})= 
	\big\{
		1+2[1-k_{\rm SE}({\bf x,y})]
		\big\}^{-\frac{H}{2}}\:,
\end{equation} as the numbers of features $n_1$ and $n_2$ both approach infinity.
\end{lemma}
\begin{proof}
The independence among the zero-mean random weights ${\bf w}_2$ and uniform variance leads to $\EX_{{\bf w}_2}[({\bf w}_2\Phi(\Omega_2{\bf h}_x))({\bf w}_2\Phi(\Omega_2{\bf h}_y))]=\Phi(\Omega_2{\bf h}_x)^T\Phi(\Omega_2{\bf h}_y)$, which at the limit $n_2\rightarrow\infty$ results in,
\begin{equation*}
	\EX[f({\bf x})f({\bf y})]\rightarrow\EX_{{\bf W}_1}\big[
	e^{-\frac{d^2({\bf x,y})}{2}}
	\big]\:,
\end{equation*}
where the squared distance between the latent outputs ${\bf h}({\bf x})$ and ${\bf h}({\bf y})$ in the exponent can be rewritten as,
\begin{align*}
	d^2({\bf x,y})&=[{\bf h(x)-h(y)}]^T[{\bf h(x)-h(y)}]\\
	&={\rm Tr}\big\{{\bf W}_1[\Phi(\Omega_1{\bf x})-\Phi(\Omega_1{\bf y})][\Phi(\Omega_1{\bf x})-\Phi(\Omega_1{\bf y})]^T{\bf W}_1^T\big\}\\
	%&={\bf W}_1{\bf G}{\bf W}_1^t
	&=\sum_{i=1}^H{\bf w}_{1,i}G{\bf w}_{1,i}^t\:,
\end{align*} where the rows of ${\bf W}_1$ are written as $\{{\bf w}_{1,1:H}\}$ in the last line. Lastly, the determinant of $(I_{2n_1}+G)$ enters as a result of
\begin{equation}
	\EX_{{\bf w}_{1,1:H}\sim\mathcal N(0,I_{2n_1})}[e^{-d^2({\bf x, y})}]=\Pi_{i=1}^H\frac{1}{\sqrt{{\rm det}[I_{2n_1}+G]}}\:. 
	%=\big\{
	%	1+2[1-k_{\rm SE}({\bf x,y})]
	%	\big\}^{-\frac{H}{2}}\:.
\end{equation} %With the above remark, we thus show the covariance of bottleneck trig network converges to
%\begin{equation}
	%\EX[f({\bf x})f({\bf y})]=\big\{
	%	1+2[1-k_{\rm SE}({\bf x,y})]
%		\big\}^{-\frac{H}{2}}
%\end{equation}
\end{proof}

%\cite{saxe2013exact}
%\cite{wilson2020bayesian}

\subsection{Deep trig net with weights representing latent function support}\label{sec:deep_mf} 
 In above subsection, the deep trigonometric net with centered and independent Gaussian weights ${\bf W}_1$ and ${\bf w}_2$ is equivalent to composition of two zero-mean GPs. In \cite{lu2021empirical}, it was shown that treating the support in latent function, i.e. $M$ hyper-data points with ${\bf h}({\bf z}_{1:M})={\bf u}_{1:M}$, as additional hyper-parameters can enhance generalization of DGPs. ${\bf z}\in\mathbb R^D$ and ${\bf u}\in\mathbb R^H$.
Here, the function space view translates to the weight parameters, ${\bf W}_1|{\bf Z,U}\sim\prod_{i=1}^H\mathcal N({\bf w}_{1,i}|{\bar{\bf w}}_i,A^{-1})$, conditional on the hyper-input and output matrices, ${\bf Z}:=\begin{psmallmatrix}{\bf z}_1,& \cdots,&{\bf z}_M\end{psmallmatrix}\in\mathbb R^{D\times M}$ and ${\bf U}=\begin{psmallmatrix} {\bf u}_1,& \cdots, &{\bf u}_M\end{psmallmatrix}\in\mathbb R^{H\times M}$, respectively. The conditional precision matrix,
\begin{equation}
	A= \big[
	I_{2n_1}+\Phi(\Omega_1{\bf Z})\Phi^T(\Omega_1{\bf Z})
	\big]
\end{equation} and the conditional mean for each output dimension,
\begin{equation}
	{\bar{\bf w}}_{1,i}=A^{-1}
	\Phi(\Omega_1{\bf Z}){\bf U}_{i,:}^T\:,
\end{equation} which can be found in Ch.2.1.2 in \cite{rasmussen2006gaussian} [also in \cite{ober2021global}].

\begin{lemma} If the latent layer weights ${\bf W}_1$ in Eq.~(\ref{deep_trig_network}) have the correlated prior distribution ${\bf W}_1\sim\prod_{i=1}^Hp({\bf w}_{1,i}|{\bf Z, U}_{i,:})$, then the covariance converges to the DGP covariance in Eq.~(\ref{DGP_covariance_f}). 
\end{lemma}
\begin{proof} The proof follows the previous one except that we are evaluating the following expectation,
\begin{align*}
	\EX[f({\bf x})f({\bf y})]&=
	\EX_{{\bf w}_{1,1:H}\sim\mathcal N({\bar{\bf w}}_{1:H},A^{-1})}\bigg[
	e^{-\frac{{\bf w}_{1,1}G{\bf w}_{1,1}^T}{2}}e^{-\frac{{\bf w}_{1,2}G{\bf w}_{1,2}^T}{2}}\cdots
	e^{-\frac{{\bf w}_{1,H}G{\bf w}_{1,H}^T}{2}}
	\bigg] \\
	&=\prod_{i=1}^H
	\frac{e^{-\frac{1}{2}{\bf \bar w}_i(I+GA^{-1})^{-1}G{\bf\bar w}_i^T}}{\sqrt{|I+A^{-1}G|}}\:,
\end{align*} in which we just focus on one term in the product. By writing the matrix $G={\bf vv}^T$ related to the inputs ${\bf x,y}$ (see Remark~7) and using the matrix inversion lemma, the exponent in above expression becomes
$-\frac{1}{2}{\bf \bar w}_i{\bf v}(1+{\bf v}^TA^{-1}{\bf v})^{-1}{\bf v}^T{\bf\bar w}_i^T$. As for the determinant in denominator, the matrix $A^{-1}G$ does not couple the vector ${\bf v}$ with its orthogonal subspace ${\bf v}_{\perp}$, leading to $|I+A^{-1}G|=1+({\bf v}^TA^{-1}G{\bf v})/({\bf v}^T{\bf v})$. With some manipulation and lengthy calculation,
\begin{align*}
	\big|I+A^{-1}G\big|&=
	1 +[\Phi_x-\Phi_y]^T[\Phi_x-\Phi_y]-[\Phi_x-\Phi_y]^T\Phi_Z[I+\Phi^T_Z\Phi_Z]^{-1}\Phi^T_Z[\Phi_x-\Phi_y]\\
	&\rightarrow 1+k_{xx}+k_{yy}-2k_{xy}-
	k_{xZ}K_{ZZ}^{-1}k_{Zx}-k_{yZ}K_{ZZ}^{-1}k_{Zy}+2k_{xZ}K_{ZZ}^{-1}k_{Zy}\:.
\end{align*} It can also be seen that the above result is identical to $(1+{\bf v}^TA^{-1}{\bf v})$. Similarly, one can show the scalar 
%\begin{equation}
${\bar{\bf w}}_i{\bf vv}^T{\bar{\bf w}}_i^T=(m_x-m_y)^2$ with the limiting form $m_x\rightarrow k_{xZ}K^{-1}_{ZZ}{\bf U}_{i,:}$. 
%We thus complete the proof.
%\end{equation}
\end{proof}

\subsection{Deep trig net with mixed spectrum features}\label{deep_sm}
The deep trigonometric networks are expressive as the choices over the weights' prior distribution are flexible. One may also consider employing different distributions over the features as we do in the shallow nets. Here, we are interested in the resultant covariance when the outer features $\Omega_2$ consist of samples from mixture of Gaussians at different centers and the inner weights ${\bf W}_1$ representing the latent function support.

\begin{lemma} When the deep trigonometric network in Eq.~(\ref{deep_trig_network}) has fixed features $\omega_2\in\mathbb R$ from samples of a mixed distribution $\sum_a\pi_a\mathcal N(\mu_a,\sigma_a^2)$ and the random variables ${\bf w}_1$ represent the weight space view of latent function support ${\bf w}_1\Phi(\Omega_1{\bf z}_{1:M})=u_{1:M}$, it is equivalent to the DGP $f({\bf x})=f_2(f_1({\bf x}))$ with $f_1|{\bf Z,u}\sim\mathcal{GP}(m,\Sigma)$ and $f_2|f_1\sim\mathcal{GP}(0,k_{\rm SM})$. $m$ and $\Sigma$ are the conditional mean and covariance matrix given the hyper-data ${\bf Z,u}$. The covariance is,
\begin{equation}
	\EX[f({\bf x})f({\bf y})]
	= \sum_a\frac{\pi_a}{(1+\sigma_a^2\delta^2)^{1/2}}\exp\big[
	-\frac{\sigma_a^2(m_x-m_y)^2+\delta^2\mu_a^2}{2(1+\sigma_a^2\delta^2)}
	\big]
	\cos\big[
	\frac{\mu_a(m_x-m_y)}{1+\sigma_a^2\delta^2}
	\big]\:.
\end{equation}
\end{lemma}
\begin{proof} It is easier to work out the covariance in the function space. Observing that 
\begin{equation*}
	\EX_{f_1|{\bf Z,u}}\big\{
	\EX_{f_2|f_1}[ f_2(f_1({\bf x}))f_2(f_1({\bf y}))
	]
	\big\} = {\rm Re}\EX_{f_1|{\bf Z,u}}[ 
	\EX_{\omega_2} e^{i\omega_2[f_1({\bf x})-f_1({\bf y})]}
	]\:,
\end{equation*} one can compute the expectation with respect to the latent function $f_1$ first, followed by that of feature $\omega_2$. Thus, we get the covariance,
\begin{equation*}
	%K_{\rm SM}(x, y)=\EX_{\omega\sim\sum\pi_a\mathcal N(\mu_a,\sigma^2_a)}\{\cos[\omega(x-y)]\}
	%K_{DSM} =
	\EX_{\omega_2\sim\sum_a\pi_a\mathcal N(\mu_a,\sigma^2_a)}\big\{
	\EX_{(f_1({\bf x}),f_1({\bf y}))^T\sim\mathcal N (m,\Sigma)}[\cos\omega_2(f_1({\bf x})-f_1({\bf y}))]
	\big\}\:,
\end{equation*} which can be computed analytically.
\end{proof}

Such deep trigonometric net is closely related to the deep kernel learning with the SM kernel~\citep{wilson2016deep}. Now it becomes clear that the outer network represents the random function $f_2\sim\mathcal{GP}(0,k_{\rm SM})$. The hyper-data ${\bf Z, u}$ constrain the inner function $f_1$, and in the limit when the hyper-data are dense the function ${\bf f}_1$ becomes deterministic~\citep{lu2021empirical}. Such situation is equivalent to passing the inputs to a parametric function and then to a GP. However, the probabilistic nature in $f_1$ in the sparse hyper-data limit is helpful for preventing overfitting in deep kernel learning with over-parameterized $ f_1$~\citep{ober2021promises}.

\section{Neural tangent kernel for trigonometric networks}\label{sec:ntk}

For probabilistic regression problems with data matrix ${\bf X}$ and observations ${\bf y}$, one has two choices over the models for prediction. The first choice is function-based models, such as GPs and DGPs. The exact GP inference produces a predictive distribution $p(y_*|{\bf X,y,x}_*)$ with closed form mean and variance that only depends on the covariance function and hyper-parameters. However, such luxury is not carried over to DGP as there is no corresponding exact inference.
%one shall expect that a different predictive distribution shall be produced, depending on which approximate DGP inference schemes are adopted. 
The second choice is weight-based models: the shallow Bayesian neural network, Eq.~(\ref{shallow_trig_network}), and its deep version, Eq.~(\ref{deep_trig_network}). For shallow trig network with fixed feature matrix $\Omega$, then it becomes a Bayesian linear regression problem (see Sec.~3), and the predictive mean and variance converge to the GP's result as the number of features $n\rightarrow\infty$. 
%Again, the Bayesian learning for the deep neural network is not analytically tractable

It is not clear whether the appealing correspondence between shallow trig network and GP with SE kernel can carry to the deep trigonometric network and 2-layer DGP discussed in this paper. Nevertheless, the perspective from neural tangent kernel~\citep{jacot2018neural,arora2019exact} may shed some light on this issue. For gradient based learning of infinite and deep neural network $f({\bf x}|\theta)$, the network function shall eventually converge to the predictive mean of GP with the following kernel,
\begin{equation}
	k({\bf x,y})=\EX_{\theta}\big[
	\frac{\partial f({\bf x}|\theta)}{\partial\theta}\cdot
	\frac{\partial f({\bf y}|\theta)}{\partial\theta}
	\big]\:,
\end{equation} where the derivative operation $\partial/\partial\theta$ with respect to all weight parameters in $\theta$ generates a vector. Moreover, the neural tangent kernel remains a constant during the gradient descent, so its value is determined by the initial distribution over $\theta$ (a recent study~\citep{seleznova2021analyzing} suggested that the neural tangent kernels for deeper model may still evolve during training). 
%associated with the deep trigonometric network in Eq. is,

Now, given the fact that the deep trigonometric network yields the same covariance as the two-layer DGP with SE kernels, it is interesting to %\pat{re-read and check?} 
derive the neural tangent kernel associated with Eq.~(\ref{deep_trig_network}), which may reveal some insights into the correspondence between deep trigonometric network and DGP. 
\begin{lemma} Assume that the features $\Omega_{1,2}$ in the deep trigonometric network in Eq.~(\ref{deep_trig_network}) are fixed and the weights ${\bf w}_{1,2}$ are learned through gradient descent with squared loss function. Then the associated neural tangent kernel reads,
\begin{equation}
	k_{\rm NTK}({\bf x,y})=k_{\rm DGP} + k_{\rm SE}k_{\rm DGP}^3\:,\label{NTK}
\end{equation} where $k_{\rm SE}$ is the SE covariance function and $k_{\rm DGP}$ is the exact covariance of the two-layer DGP. Note that we have set all hyper-parameters to unit for ease of notation.
\end{lemma}
\begin{proof}
As only the weight parameters are learned, the neural tangent kernel has the following expression, 
\begin{align*}
	k_{NTK}({\bf x,y})&=\EX_{{\bf W}_1}\bigg\{\EX_{{\bf W}_2|{\bf W}_1}\big[
	\frac{\partial f({\bf x})}{\partial {\bf W}_2}
	\frac{\partial f({\bf y})}{\partial {\bf W}_2}+
	\frac{\partial f({\bf x})}{\partial {\bf W}_1}
	\frac{\partial f({\bf y})}{\partial {\bf W}_1}
	\big]
	\bigg\}\\
	&=K_{\rm DGP}+K_e\:,
\end{align*} where we observe that the first term (derivative wrt second weight ${\bf w}_2$) is the same as the covariance of DGP (see Sec. 4.1).
Next, we shall focus on the second term, $K_e$. Notice that the order of differentiation $\partial f/{\bf w}_1$ and the expectation $\EX_{{\bf w}_2|{\bf w}_1}$ can be switched. To facilitate the computation, we can temporarily write $f({\bf x})={\bf w}_2\Phi(\Omega_2{\bf w}_a\Phi(\Omega_1{\bf x}))$ and $f({\bf y})={\bf w}_2\Phi(\Omega_2{\bf w}_b\Phi(\Omega_1{\bf y}))$ so that we can first compute the expectation and then take the derivatives. The rest of derivations just rest on some simple tricks, 
\begin{align*}
	K_e &= \EX_{{\bf w}_1}\bigg\{
	\sum_{i=1}^{n_1}\frac{\partial^2}{\partial w_{a,i}\partial w_{b,i}}
	\EX_{{\bf w}_2|{\bf w}_{a,b}}[f_a({\bf x})f_b({\bf y})]\big|_{{\bf w}_a={\bf w}_b={\bf w}_1}
	\bigg\} \\
	& = \EX_{{\bf w}_1}\bigg[ 
	e^{-\frac{{\bf w}_1G{\bf w}_1^T}{2}}\Phi^T(\Omega_1{\bf x})\Phi(\Omega_1{\bf y})(1-{\bf w}G{\bf w}_1^T) 
	\bigg]\\
	& = e^{-\frac{|{\bf x-y}|^2}{2}}(1+2\frac{\partial}{\partial \lambda}) \EX_{{\bf w}_1}\big[ e^{-\lambda\frac{{\bf w}_1G{\bf w}_1^T}{2}}
	\big]\big|_{\lambda=1}\\
	& = k_{{\rm SE}}({\bf x,y})\big[1+2(1-k_{\rm SE}({\bf x,y}))\big]^{-\frac{3}{2}}\:.%\\
	%& = k_{\rm SE}k_{DGP}^3\:.
\end{align*} To arrive at the second equality, we have used 
\begin{equation*}
	\sum_i\partial^2_{w_{a,i}w_{b,i}}e^{-\frac{1}{2}({\bf w}_a\cdot\Phi_x-{\bf w}_b\cdot\Phi_y)^2} =
	e^{-\frac{1}{2}({\bf w}_a\cdot\Phi_x-{\bf w}_b\cdot\Phi_y)^2}
	(\Phi_x\cdot\Phi_y)\big[1-({\bf w}_a\cdot\Phi_x-{\bf w}_b\cdot\Phi_y)^2
	\big]\:.
\end{equation*} %and in the third line the expectation reads $\EX_{\bf w}[\exp(-\frac{\lambda}{2}{\bf w}G{\bf w}^T)]=$
\end{proof}

%The two kernels $k_{\rm NTK}$ and $k_{\rm DGP}$ are different.

As for the shallow trig net in Eq.~(\ref{shallow_trig_network}), it is easy to show that the NTK is the same as $k_{\rm SE}$ if $\Omega$ are independent and normal. Hence, the predictive distribution for $y_*|{\bf x}_*,\Omega,{\bf X,y}$ is the same for the shallow Bayesian trig network in wide limit and GP with SE kernel. Moreover, the mean of this distribution shall coincide with the prediction obtained using gradient descent as the equivalence between NTK and $k_{\rm SE}$ suggests. However, the correspondence is much intriguing between DGP and deep trigonometric network as there is no exact inference for both models. If one adopts the moment matching inference~\citep{lu2020interpretable} which treats the marginal prior distribution of DGP as a GP~\citep{lu2021empirical}, then the predictive distribution is the same as the GP with $k_{\rm DGP}$. With the equivalence between DGP and deep trigonometric net, one can say that the single prediction made by gradient descent algorithm shall converge to the predictive mean of a GP with $k_{\rm NTK}$ in Eq.~(\ref{NTK}). The origin for the discrepancy between $k_{\rm DGP}$ and $k_{\rm NTK}$ is a very interesting question as the exact DGP inference is impossible and the optimization of deep trigonometric network is not convex. %non-Gaussian nature of DGP may attribute to it.

%\cite{hershey2007approximating}

\section{Finite width corrections}\label{sec:finite}
For both the shallow and deep trig networks, their output $f({\bf x})$ depend on two sets of parameters: the weights ${\bf W}$'s and the projections $\Omega$'s. We have connected them with shallow GPs and deep GPs, respectively. By treating the layer widths to be infinity, we have obtained the limiting kernel $k_{\rm DGP}$ and neural tangent kernel $k_{\rm NTK}$ for the deep trig network. Here, we shall consider the deviation from the limiting kernels 
%$k_{\rm DGP}$ and neural tangent kernel $k_{\rm NTK}$ 
when the layer width is large but finite. An important implication is that the kernel only converges to its fixed and limiting form when the inner width $n_1$ is infinite, suggesting that the inner layer is more relevant to the feature learning than the outer one.
%Furthermore, it shall be seen that the outer width $n_2$ and inner width $n_1$ play different roles in how the limiting kernels are reached.  
%\subsection{Kernel}

We follow~\citep{yu2016orthogonal} and define the kernel estimator, $\hat k_{\rm DGP}({\bf x,y}):=\EX_{{\bf w}}[f({\bf x})f({\bf y})|\Omega]$, for the deep net. With some simple algebra, 
%which retains the explicit dependence on the projection parameters $\Omega$'s. After some algebra, it can be seen that  
%Similar with the analysis of shallow network in~\citep{yu2016orthogonal}, we define the the kernel estimator as the correlation conditioned on the projection parameters,
\begin{align}
	\hat k_{\rm DGP} &=\frac{1}{n_2}{\rm Re}\sum_i\EX_{{\bf W}^{(1)}}\{\prod_{k,m}\exp[i\omega^{(2)}_{ik}w^{(1)}_{km}(\Phi_{\bf x}-\Phi_{\bf y})_m]\}\\
	&=\frac{1}{n_2}\sum_{i=1}^{n_2}\exp\{
	-\sigma_w^2\sum_{k=1}^H[\omega^{(2)}_{ik}]^2\cdot
	\frac{1}{n_1}\sum_{m=1}^{n_1}
	\big[1-\cos\Omega_m^{(1)}\cdot({\bf x-y})
	\big]
	\}\label{eq:kernel_estimator}
\end{align} where the components of post-activation vector read $\Phi_{\bf x}=(1/\sqrt{n_1})[\cos\Omega^{(1)}_{1:n_1}\cdot{\bf x},\sin\Omega^{(1)}_{1:n_1}\cdot{\bf x}]$, and the above second equality follows from the fact that weights $w^{(1)}\sim\mathcal N(0,\sigma_w^2)$ are iid. The two summations are over the projection parameters $\omega^{(2)}$ in outer layer and projection vectors $\Omega^{(1)}$ in inner layer. Now the estimator $\hat k_{\rm DGP}$ depends on the realizations of $\Omega^{(1)}_{1:n_1}$ and $\Omega^{(2)}_{1:n_2}$.

%Now it can be seen that the average of random variables, $\bar s:=(1/N_1)\sum_m[1-\cos\Omega_m\cdot({\bf x-y})]$, becomes deterministic in the limit $N_1\rightarrow\infty$. Therefore, the kernel estimator does converge to $k_{\rm DGP}$ when $N_2\rightarrow\infty$ too. 

\begin{lemma} When the latent dimension $H$ is finite and the inner layer width $n_1$ is large but finite, the mean of kernel estimator for the deep trig network approximately reads,
\begin{equation}
	\EX_{\Omega}[\hat k_{\rm DGP}({\bf x,y})]\approx[\int e^{-(1-\hat k_{\rm SE})\omega^2\sigma^2_w}
	d\mu(\omega)
	d\mu(\hat k_{\rm SE})]^H 
	%= [1-(1-k_{{\rm SE}})+\cdots+\frac{1}{N_1}()+\cdots]^H\:,
\end{equation} with the normal $\omega$ representing iid entries in $\Omega^{(2)}_{1:n_2}$ and $\hat k_{\rm SE}:=(1/n_1)\sum_m\cos\Omega^{(1)}_m\cdot({\bf x-y})$. Here $d\mu(\hat k_{\rm SE})$ takes the approximate density $\mathcal N(\mu_s,\sigma^2_s)$ with mean $\mu_s:=\EX_{\Omega^{(1)}}[\hat k_{\rm SE}]$ and variance $\sigma^2_s:={\rm Var}_{\Omega^{(1)}}[\hat k_ {\rm SE}]$.   
%the $d\mu$'s denoting the densities for weight $w$ and $\bar s$. 
%In infinite width limit $N_1\rightarrow\infty$, the above expectation coincides with the limiting kernel $[1+2(1-k_{\rm SE})]^{-H/2}$ when the two inputs ${\bf x}$ and ${\bf y}$ are close. 

%Similarly, the variance for the kernel estimator reads,
%\begin{equation}
%	{\rm Var}[\hat k_{\rm DGP}({\bf x,y})]\approx\frac{1}{N_2}\big\{
%	[\int e^{-2(1-\hat k_{\rm SE})\omega^2\sigma^2_w}d\mu]^H -
%	[\int e^{-(1-\hat k_{\rm SE})\omega^2\sigma^2_w}d\mu]^{2H}
%	\big\}\:.  
%\end{equation} 
\end{lemma}
\begin{proof}
First, rewriting the expectation of some smooth function $\alpha$ as $\EX_{\Omega^{(1)}}[\alpha(\hat k_{\rm SE})]=\EX_{\hat k_{\rm SE}}[\alpha(\hat k_{\rm SE})]$ is valid so one can apply it to Eq.~(\ref{eq:kernel_estimator}) as well. Next, $\hat k_{\rm SE}$ has mean $\mu_s=k_{\rm SE}$ and variance $\sigma^2_s=(1-k_{\rm SE}^2)^2/(2n_1)$ if $\Omega^{(1)}$ is normal~\citep{yu2016orthogonal}. For large but finite $n_1$, the central limit theorem suggests that $\hat k_{\rm SE}$ can be treated as a Gaussian. Lastly, the iid and normal assumption of entries in $\Omega^{(2)}$ result in the product form.
\end{proof}

A few observations follow from the lemma. First, when $n_1$ is infinite, the random variable $\hat k_{\rm SE}$ becomes deterministic as $\sigma_s^2$ vanishes [\citep{lee2018deep} employed similar strategy in proving GP behavior for DNNs]. Thus the density $d\mu(\hat k_{\rm SE})$ approaches a delta function and the remaining integration over $\omega$ results in $\EX[\hat k_{\rm DGP}]=k_{\rm DGP}$. Note that, due to the randomness in $\omega$, ${\rm Var}[\hat k_{\rm DGP}]$ does not vanish, which signifies the difference with NNGP. Secondly, when the latent dimension $H$ is also infinite and when the weight variance has $\sigma^2_w=1/H$, then the term $\sigma^2_w\sum_{k=1}^H[\omega^{(2)}_{ik}]^2$ summing over squared projection parameters in outer layer in Eq.~(\ref{eq:kernel_estimator}) also converges to its fixed mean, which in turn leads to $\EX[\hat k_{\rm DGP}]=\exp[k_{\rm SE}-1]$ along with vanishing ${\rm Var}[\hat k_{\rm DGP}]$. This limiting kernel first appeared in~\citep{duvenaud2014avoiding} discussing asymptotic kernel of DNNs, while it corresponds to the case when the variances in $\hat k_{\rm SE}$ and $\overline{\omega^2}$ both vanish.  
%The vanishing variance was also exploited for proving GP behavior for infinite wide DNN.
%Now it is also clear that even with very large $n_1$ the finite $H$ prevents $\hat k_{\rm DGP}$ from collapsing to the fixed $k_{\rm DGP}$ in a statistical sense. Moreover, we believe the non-Gaussian statistics of DGP and bottleneck DNN has to do with the randomness in $\hat k_{\rm DGP}$.

As for finite $n_{1,2}$ and $H$, one can proceed to show $\EX[\hat k_{\rm DGP}]=\langle[1+2\sigma_w^2(1-\hat k_{\rm SE})]^{-1/2}\rangle^H$ after marginalizing the entries in $\Omega^{(2)}_{1:n_2}$. The brackets $\langle\cdot\rangle$ stands for averaging wrt the random variable $\hat k_{\rm SE}$. However, even with $\hat k_{\rm SE}$ approximately being a Gaussian, the mean does not have a closed form. %, which is true for ${\rm Var}[\hat k_{\rm DGP}]$ too. 
%The statistics associated with $\hat k_{\rm DGP}$ is not tractable as its mean and variance are not tractable given the fact that $\hat k_{\rm SE}$ in general is not Gaussian. 
Nevertheless, we again employ the Gauss-Hermite quardature method to approximate the integration. The following remark summarizes the deviation from the limiting $k_{\rm DGP}$ due to the finite width $n_{1,2}$. 
%one can focus on two close inputs ${\bf x}$ and ${\bf y}$ so that $(1-\hat k_{\rm SE})\ll1$ is plausible. The following remark summarizes the deviation from the limiting $k_{\rm DGP}$ due to the finite width $n_{1,2}$.
%From the perspective of $\hat k$, the deep trig network with random projections can be approximately viewed as a GP with a random kernel which converges to the limiting one almost surely when the latent width $H$ and widths $N_{1,2}$ are simultaneously infinite. The non-Gaussian statistics associated with the distribution over output functions also vanishes in the wide limit.

\begin{remark}
Consider $H=1$, one can show the approximate deviation yields, 
\begin{equation}
	|k_{\rm DGP}-\EX_{\Omega}[\hat k_{\rm DGP}]|
	%\approx\frac{3}{2}\sigma^2_{\bar s}
	\approx
	%\frac{3\sigma_w^2}{4n_1}(1-k_{\rm SE}^2)^2\:.
	%(1-\frac{\lambda_0+2\lambda_1}{\sqrt{\pi}})k_{\rm DGP}-
	\frac{3\lambda_1z_1^2\sigma_w^4}{n_1\sqrt{\pi}}(1-k_{\rm SE}^2)^2k_{\rm DGP}^3\:,
\end{equation} in which the values of Gauss-Hermite parameters $\lambda_{0,1}$ and $z_1$ are listed in~\citep{greenwood1948zeros}. 
\end{remark}
\begin{proof}
Considering contributions from the three roots $\{z_0,z_{\pm}\}$ of the third order Hermite polynomial, the approximation of integral reads $\EX[\hat k_{\rm DGP}]=\sum_{i=0,\pm 1}\frac{\lambda_i}{\sqrt{\pi}}[1+2\sigma_w^2(1-k_{\rm SE}+\sqrt{2}\sigma_sz_i)]^{-1/2}$, which, for zeroth order of $\sigma_s$, gives $\EX[\hat k_{\rm DGP}]=(\lambda_0+2\lambda_1)k_{\rm DGP}/\sqrt{\pi}$ where the fact $\lambda_1=\lambda_{-1}$ is used. The next order of correction is $O(\sigma_s^2)$ due to the symmetry $z_1=-z_{-1}$ and the expansion $(1+\epsilon)^{-1/2}=1-\epsilon/2+3\epsilon^2/8+\cdots$. One can thus recover the above expression if one further takes $(\lambda_0+2\lambda_1)/\sqrt{\pi}\approx0.99918$ to be unity. 
%In addition, considering higher order Hermite polynomials and the roots shall results in closer zeroth order term. 
\end{proof}
%Similar calculation results in the variance, 
%which leads to the observation that for $H=1$ and ${\bf x}$ close to ${\bf y}$,
%\begin{equation}
%	{\rm Var}[\hat k_{\rm DGP}]\approx
%	n_2^{-1}\big[
%	\frac{1}{\sqrt{1+4\sigma_w^2(1-k_{\rm SE}})} -
%	\frac{1}{1+2\sigma_w^2(1-k_{\rm SE})}
%	\big]\:.
	%\frac{2\sigma_w^4}{n_2}(1-k_{\rm SE})^2+\frac{3\sigma_w^4}{2n_1n_2}(1-k_{\rm SE}^2)^2\:.
%\end{equation} 
%We have for simplifying both expressions, and the higher order terms are neglected.
%\end{remark}
%The proof follows from the expansion $1/\sqrt{1+\epsilon}\approx1-\epsilon/2+3\epsilon^2/8$ and from the defition of variance.
It is interesting to note from the minimum deep model the nontrivial effect of depth on statistics of $\hat k_{\rm DGP}$. For the shallow model in~\citep{yu2016orthogonal}, the mean coincides with the {\it fixed} kernel, i.e. $\EX[\hat k_{\rm SE}]=k_{\rm SE}$.  %and the variance decreases with $1/n$. 
In contrast, $\EX[\hat k_{\rm DGP}]\neq k_{\rm DGP}$ when the inner width $n_1$ is not infinite, which implies that the inner layer is more relevant to feature learning than the outer one. \cite{aitchison2020bigger} had similar observation in a two-layer linear Bayesian model.
%which clearly show the nontrivial effect of being deep in a minimum model!
%shows the effect of nonlinear activation contains the correction due to the finite width.

%As for the variance of kernel, which is defined as ${\rm Var}[k_{\rm DGP}]=\EX[f({\bf x})^2f({\bf y})^2]-\EX[f({\bf x})f({\bf y})]^2$, one can follow the same spirit in computation. Here, we only 
%\subsection{Neural tangent kernel}

The same formulation can be applied to analytically investigate the finite-width effect on NTK. 
%The case of deep ReLu network was studied in~\citep{hanin2019finite} but with a rather different approach. 
After some manageable algebra, we can arrive the following estimator for NTK,
\begin{equation}
	\hat k_{\rm NTK}=%\hat k_{\rm DGP} + 
	%\big[
	\frac{1}{n_2}\sum_{i=1}^{n_2}
	\{
	1+[\omega^{(2)}_i]^2+\frac{\partial}{\sigma_w^2\partial\lambda}
	\}
	e^{-\lambda\sigma_w^2[\omega^{(2)}_i]^2(1-\hat k_{\rm SE})}
	\big|_{\lambda=1}\:,
\end{equation} for $H=1$. The deviation $|\hat k_{\rm NTK}-k_{\rm NTK}|\approx(6\lambda_1/\sqrt{\pi})(\sigma_s^2\sigma_w^2z_1^2k_{\rm DGP}^2)(1+2\sigma_w^2)\propto(1/n_1)$ can be obtained by similar computations. The NTK case of deep ReLu network was studied in~\citep{hanin2019finite} but with a rather different approach and assumption. 

%The variance of kernel ${\rm Var}[k_{\rm DGP}]=\EX[f({\bf x})^2f({\bf y})^2]-\EX[f({\bf x})f({\bf y})]^2$ 

\section{Simulations}\label{sec:simulation}

In this paper, an important consequence of the translation between DGP in weight representation and function representation is that one can pursue the MAP estimate of weight parameters from the exact posterior. The point estimate then allows to obtain the mean of predictive prediction, which does not seem possible with a function space approach. Another interesting perspective is to compare the predictive means with those obtain from NTK regression, which corresponds to the gradient-based learning with an infinitesimal learning rate.

The flexibility of DGP makes data fusion and multi-fidelity regression possible~\citep{cutajar2019deep,lu2021conditional}. The translation, which also includes log of the correlated prior over weights, then allows the neural network version of DGP multi-fidelity regression model. In such case, the regularizer contains the term $-\log p({\bf w}_1|{\bf z},u)$, indicating the correlation between the components and the mode $\overline{\bf w}_1$ as a representation of low-fidelity data $\{{\bf z,u}\}$ in weight space.

%In the following subsections, we shall demonstrate the MAP prediction 
Lastly, the analysis of shallow trig nets in Sec.~\ref{sec:shallow} suggests that the expressive power may be enhanced with i) adopting different weight prior distributions, which is equivalent to different GP kernels for function space regression, and ii) inserting phase networks before entering the sine/cosine activation units, which in principle removes the Gaussianity of the marginal prior distribution. Below, numeric simulations on real-world and toy data are present to support our findings. 

%Up to this point, we have investigated the marginal prior for shallow trig nets and the covariance of two-layer trig nets. Basically, this enables the translation from GP and DGP in function space to weight space. For shallow nets, random projections from mixture of Gaussians and inserting a phase shift network before activation are suggested to increase the expressivity. More importantly, two implications on DGP due to the translation follow. First, the moment matching kernel of DGP and the NTK have fixed functional form, independent of the training data. With the translation to trig network, we are able to see the finite width effects of inner and outer layers on prediction as the weights are learned via the back propagation of training error. Secondly, the MAP estimate of weights in deep wide trig nets shall correspond to the exact DGP MAP in function space when the widths become infinite. While MAP does not offer uncertainty estimation, it is interesting to compare the regression results among the deep trig nets and kernel regression using moment matching kernel and NTK. Lastly, multi-fidelity regression is realizable due to the flexible DGP structure. The translation along with treating $\log p({\bf W}_1|{\bf z},u)$ as the regularizer for the inner layer makes the network-based multi-fidelity regression possible.

\subsection{Approaching the exact predictive mean with deep trig nets} 
Here, we are interested in predicting the trend of carbon dioxide concentration in Mauna Loa data set. It is well known that the GP regression with SE kernel fails to capture the short time scale variation as the prior density has its mass concentrated on smooth functions. We implemented using PyTorch the moment matching kernel correspond to the two-layer DGP with both kernels being SE~\citep{lu2020interpretable} and the corresponding NTK derived in Sec.~\ref{sec:ntk}.
%As this is the classic example for kernel learning in~\citep{}
The GP kernel regression with these two kernels (left: moment matching SE[SE] kernel, right: NTK) is shown in Fig.~\ref{kernel_result}, in which the two results are only slightly different. The fixed form of kernels and the learned length scale $\ell_1\ll 1$ in first layer leads to the constant predictive mean in the extrapolation. The fact that the rapid variation present in the training data is learned but not generalized can be considered a symptom of lack of feature learning.  
%Although the rapid variation present in the training data (blue dots) is captured in the interpolation regime, it is not carried over to the extrapolation regime 

\begin{figure}[ht]
\begin{center}
%\framebox[4.0in]{$\;$}
\includegraphics[width=0.36\linewidth]{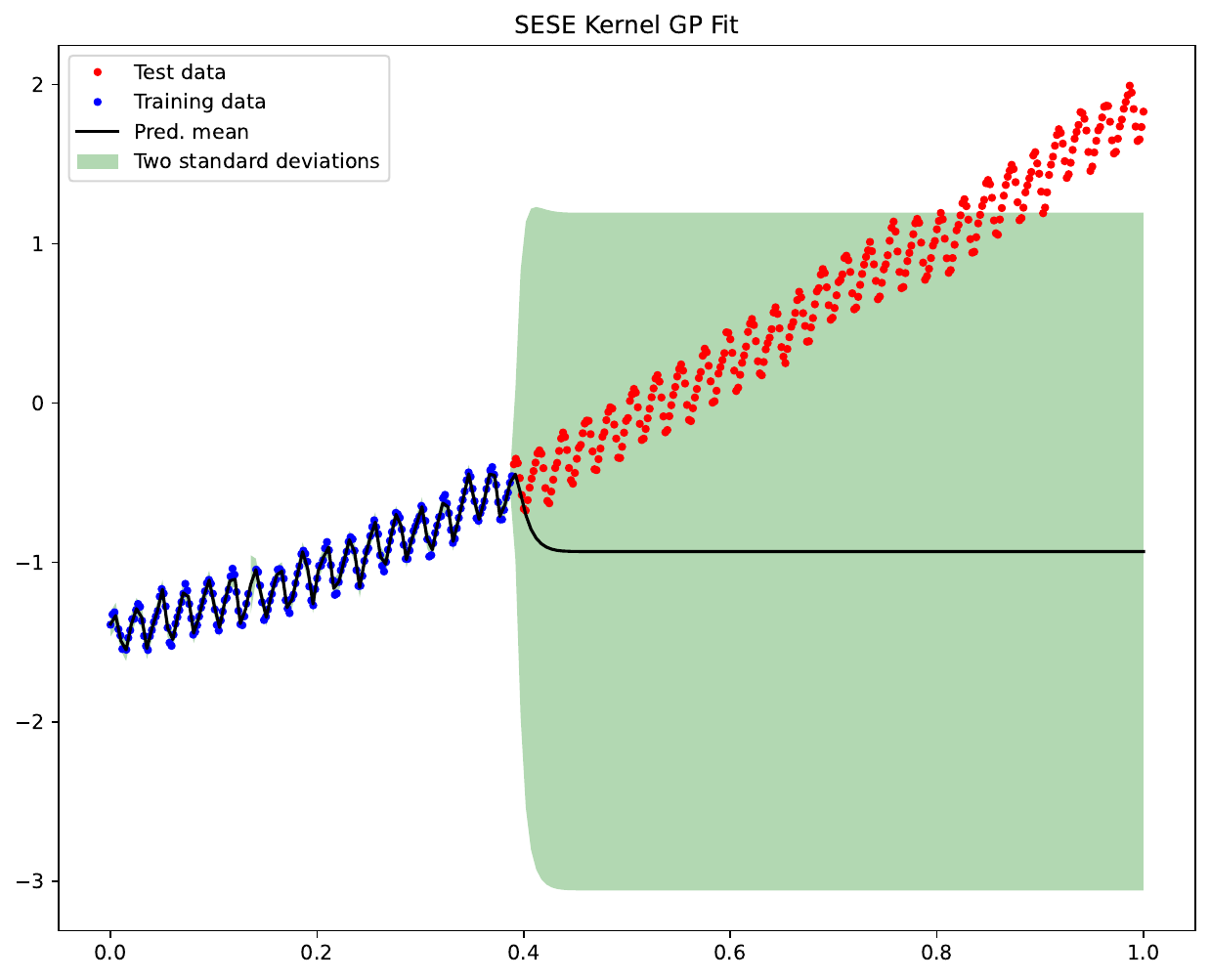}
\includegraphics[width=0.36\linewidth]{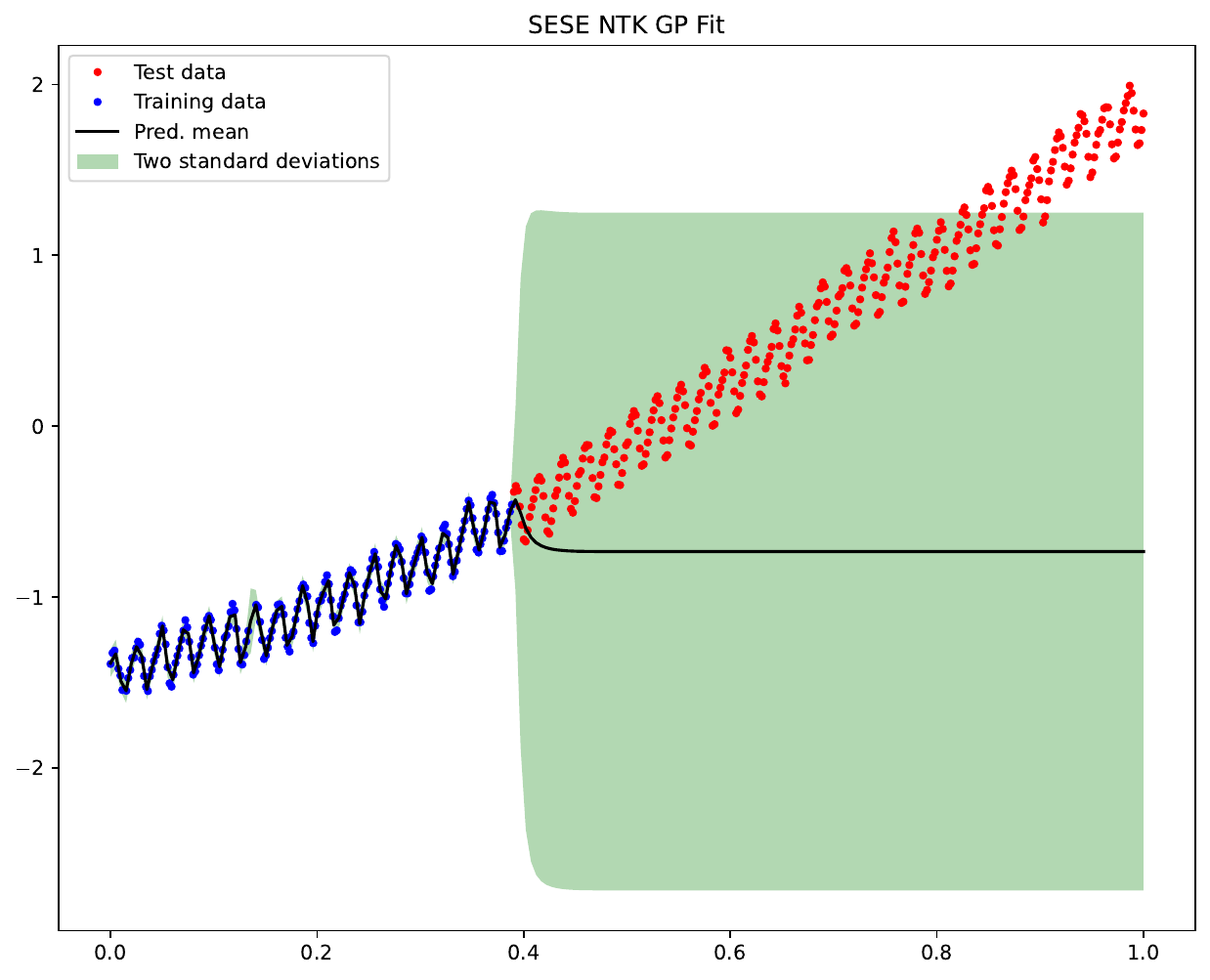}
\end{center}
\caption{GP fitting the standardized carbon dioxide concentration data. {\bf Left}: with the moment matching SE[SE] kernel. {\bf Right}: with NTK.}
\label{kernel_result}
\end{figure}

With the translation from function space to weight space representation for DGP, it is interesting to apply the gradient-based learning for prediction. The two-layer DGP is then transformed into the two-layer trig network. We consider the squared loss together with the standard quadratic regularizer as the objective, %[Eq.~(\ref{objective_1})]. 
\begin{equation}
	\mathcal L = \sum_i[y_i-f({\bf x}_i)]^2 + \lambda{\bf W}^t{\bf W}\label{objective_1}\:.
	%[\log p({\bf w}_2)+\log p({\bf W}_1)]
\end{equation} Here, ${\bf W}$ stand for the flattened collection of weight parameters within the two layers, corresponding to the fact that the all the weights have independent and zero-mean Gaussian as prior. As for the random frequencies $\Omega_{1,2}$, they are samples from $\mathcal N(0,1/\ell_1^2)$ and $\mathcal N(0,1)$, respectively, and we kept them fixed in the process of gradient learning. 

The two-layer trig network can have variation in the widths $n_{1,2}$ and the bottleneck width $H$, respectively. Fig.~\ref{network_result} shows the predictive means obtained with three variations in the network structure. Left panel displays the results from running with the six structures, namely $(n_1,H,n_2)=(2^{4:9},1,300)$. Middle panel is for $(n_1,H,n_2)=(300,1,2^{4:9})$, and right panel is for $(n_1,H,n_2)=(300,2^{0:7},300)$. 
%The corresponding training losses as function of these widths are displayed in the bottom row.
%The corresponding training errors for all three sets of variations can be found in panel (d).

\begin{figure}[ht]
\begin{center}
\includegraphics[width=0.3\linewidth]{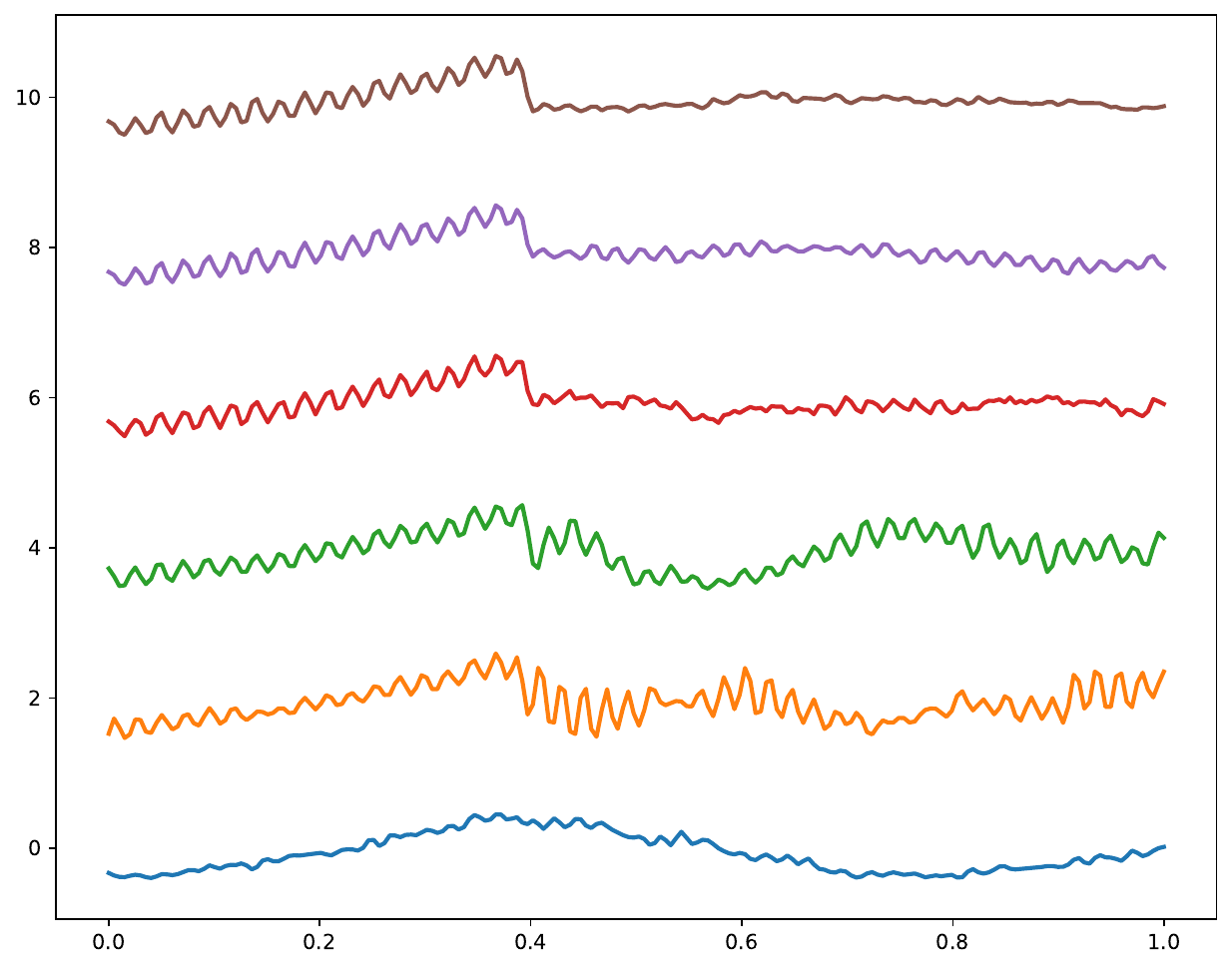}
\includegraphics[width=0.3\linewidth]{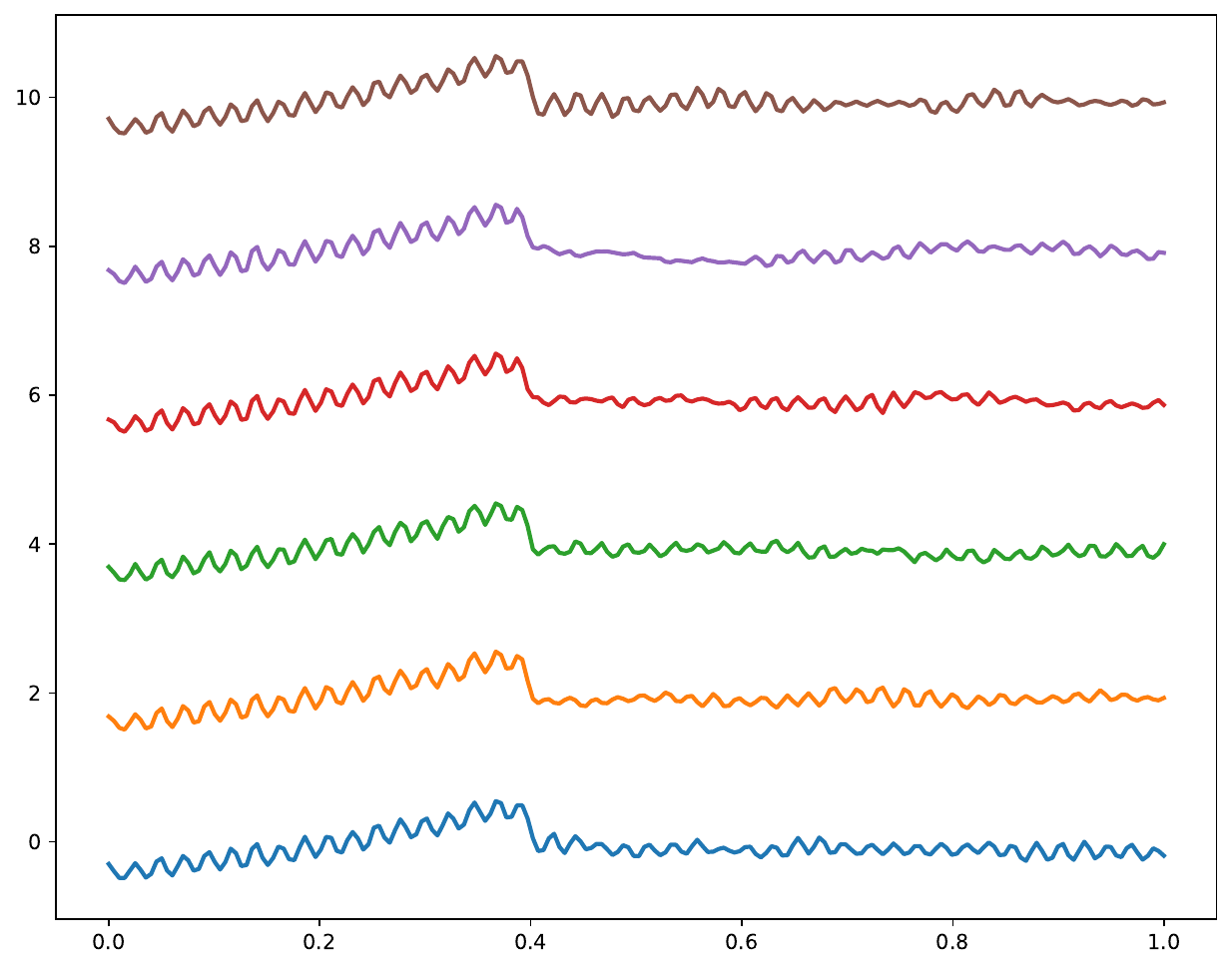}
\includegraphics[width=0.3\linewidth]{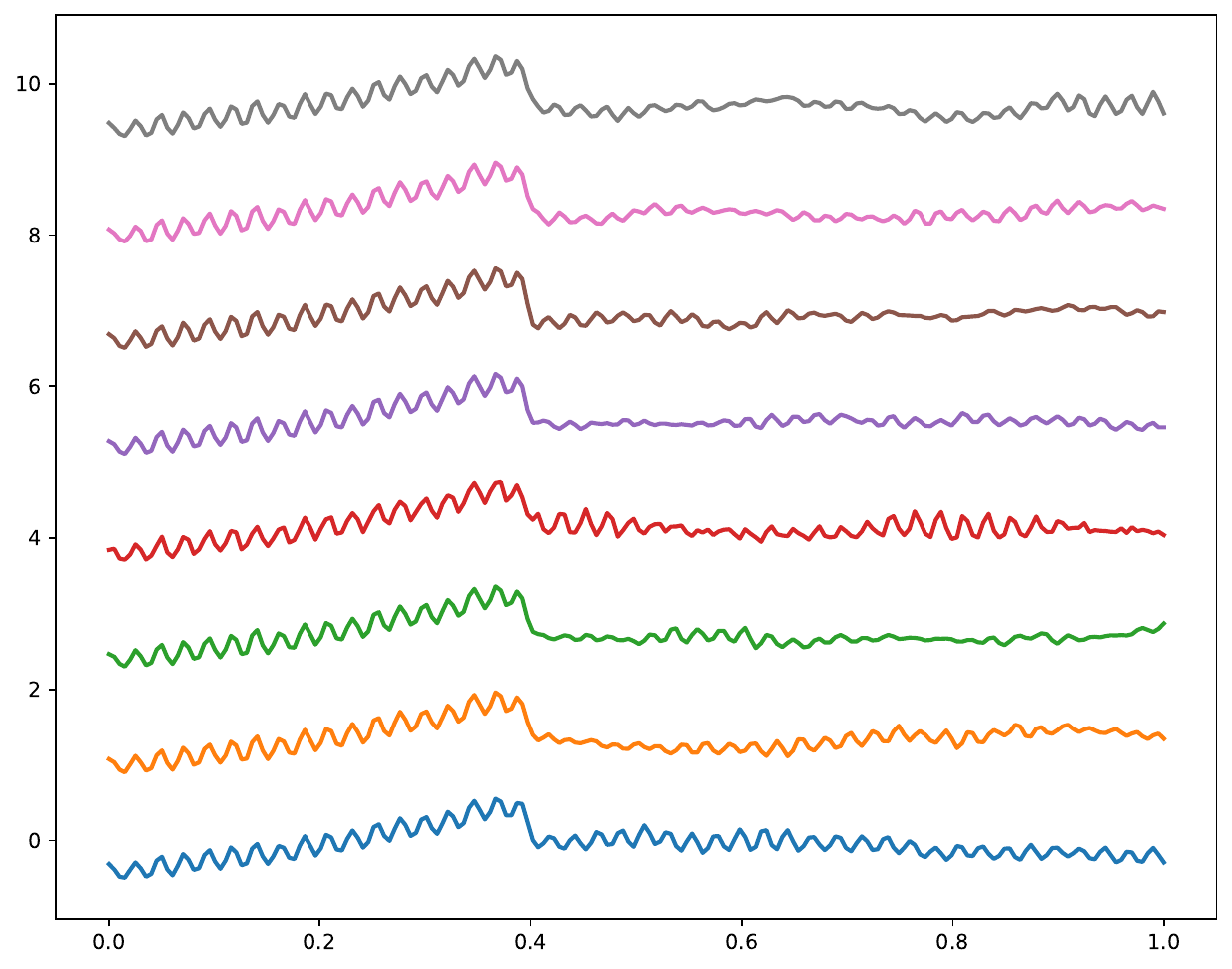}
%\includegraphics[width=0.3\linewidth]{Figure/Loss_w1a.pdf}
%\includegraphics[width=0.3\linewidth]{Figure/Loss_w2a.pdf}
%\includegraphics[width=0.3\linewidth]{Figure/Loss_ba.pdf}
%\includegraphics[width=0.36\linewidth]{Figure/WidthEffect_trainloss.pdf}
%\fbox{\rule[-.5cm]{0cm}{4cm} \rule[-.5cm]{4cm}{0cm}}
\end{center}
\caption{Predictive means from gradient-based learning with three sets of variation in network structures specified by the three widths $(n_1,H,n_2)$. We vertically shift these results for better visualization. Left panel displays the six results from $(n_1,H,n_2)=(2^{4:9},1,300)$, middle panel for $(n_1,H,n_2)=(300,1,2^{4:9})$, and right panel for $(n_1,H,n_2)=(300,2^{0:7},300)$. 
%The corresponding training errors are shown in panel (d). 
%{\bf Bottom} row: the corresponding training loss as function of the width.
}\label{network_result}
\end{figure}

A few observations follow. The analysis in Sec.~\ref{sec:finite} suggests that deep trig network with the structure $(n_1\rightarrow\infty,H=1,n_2<\infty)$ still converge to the limiting kernel $k_{\rm DGP}$. This is in contrast to the structure $(n_1<\infty,H=1,n_2\rightarrow\infty)$ leading to a deviation $\propto1/n_1$ from the limiting kernel. Therefore, the inner width $n_1$ plays a more critical role in learning than $n_2$. In the left panel of Fig.~\ref{network_result}, it is seen that when $n_1\geq64$ (green and above) the rapid variation in training data is learned. This {\it feature} is carried over to the future times, but as $n_1$ increases the result is more close to that in Fig.~\ref{kernel_result}. 
%The networks with even narrower $n_1$ (blue, orange, and green) seem to struggle to learn both the rapid variation and slowly growing trend in training data. 
In the middle panel, the outer width $n_2$ does not seem to have effect on the learning and the generalization. Then, the variation in $H$ theoretically signifies the transition from DGP behavior to GP~\citep{pleiss2021limitations}. In right panel, however, we do not see significant difference by varying the bottleneck width.
%the one with the largest $H$ (brown) only has a very weak signal carrying the fast oscillation in comparison with the rest. 

With the weight representation of the two-layer zero-mean DGP, we are able to approach the exact mean of intractable predictive distribution with the finite-width deep trig nets. Comparing with the kernel composition trick~\citep{duvenaud2013structure} and the designed activation units~\citep{pearce2020expressive}, we may conclude that, for this particular data, simply stacking two vanilla GPs into a DGP does not excel in enhancing the expressivity. 
%In addition, although varying $H$ does not seem to undermine learning the training data, the increase in training errors [green curve in panel (d)] might arise from the increasing number of weights in first layer appeared in the regularizer.

\subsection{Toy multi-fidelity regression}

DGP is a flexible prior exploiting the expressive power in compositionality, and an ideal model for fusing data from different levels of precision~\citep{cutajar2019deep}. Given the two-fidelity data $\{{\bf X}_1,{\bf y}_1\}$ (plentiful but low fidelity) and $\{{\bf X}_2,{\bf y}_2\}$ (rare but high fidelity), we may model the regression as inferring the composite function $f(x)=h(g(x))$ and the data are treated as observations, namely ${\bf y}_1=g({\bf x}_1)$ and ${\bf y}_2=f({\bf x}_1)$. It was shown in~\citep{lu2021conditional} that the moment matching kernel in Eq.~(\ref{DGP_covariance_f}) which takes the low fidelity data as the support for latent function $g(x)$ can reasonably well recover the truth function $f(x)$ even though the high-fidelity training data is rare. In the left panel of Fig.~\ref{multifidelity}, we reproduced the simulation result in~\citep{lu2021conditional} with a PyTorch-based implementation.

\begin{figure}[ht]
\begin{center}
%\framebox[4.0in]{$\;$}
\includegraphics[width=0.36\linewidth]{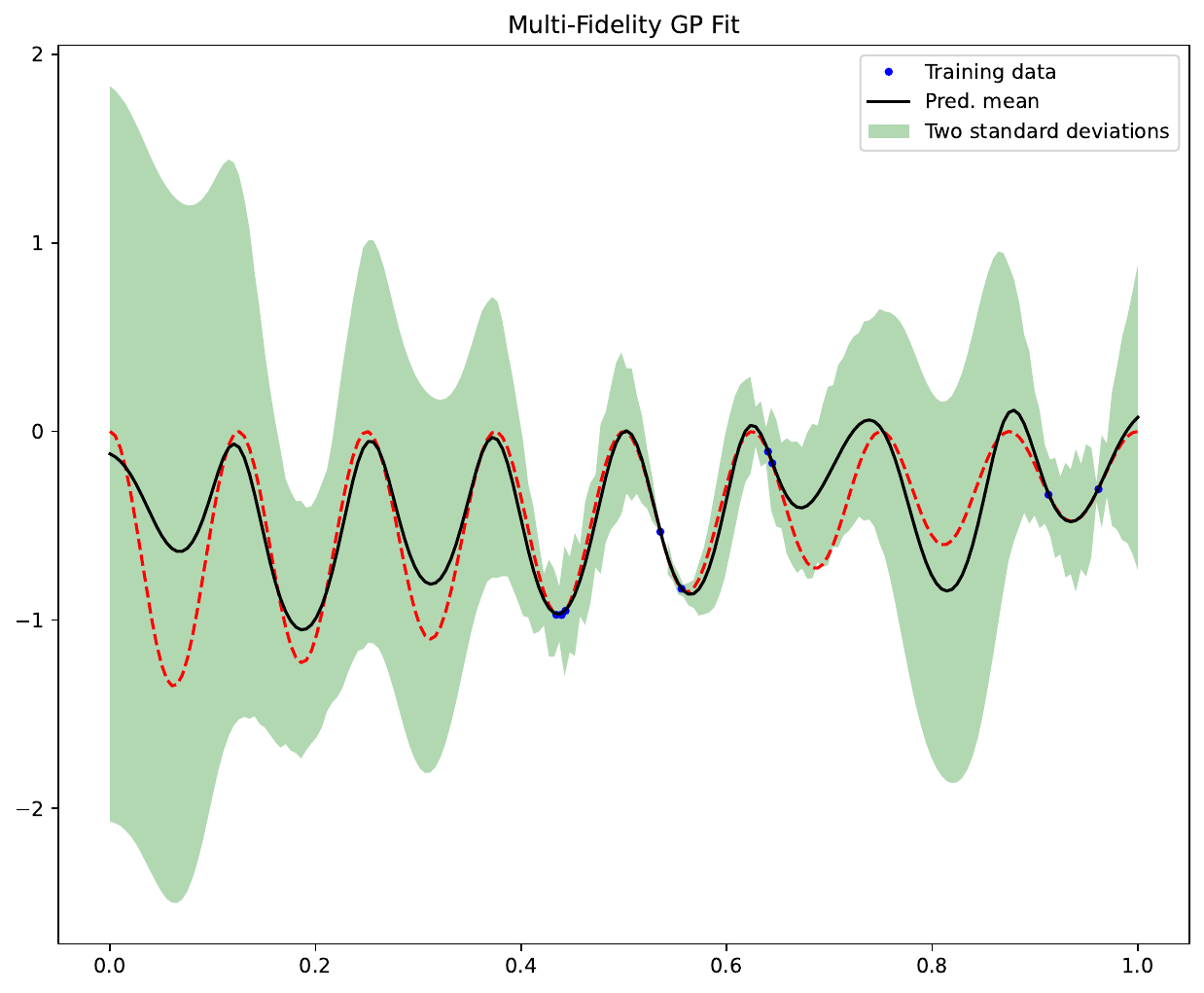}
\includegraphics[width=0.38\linewidth]{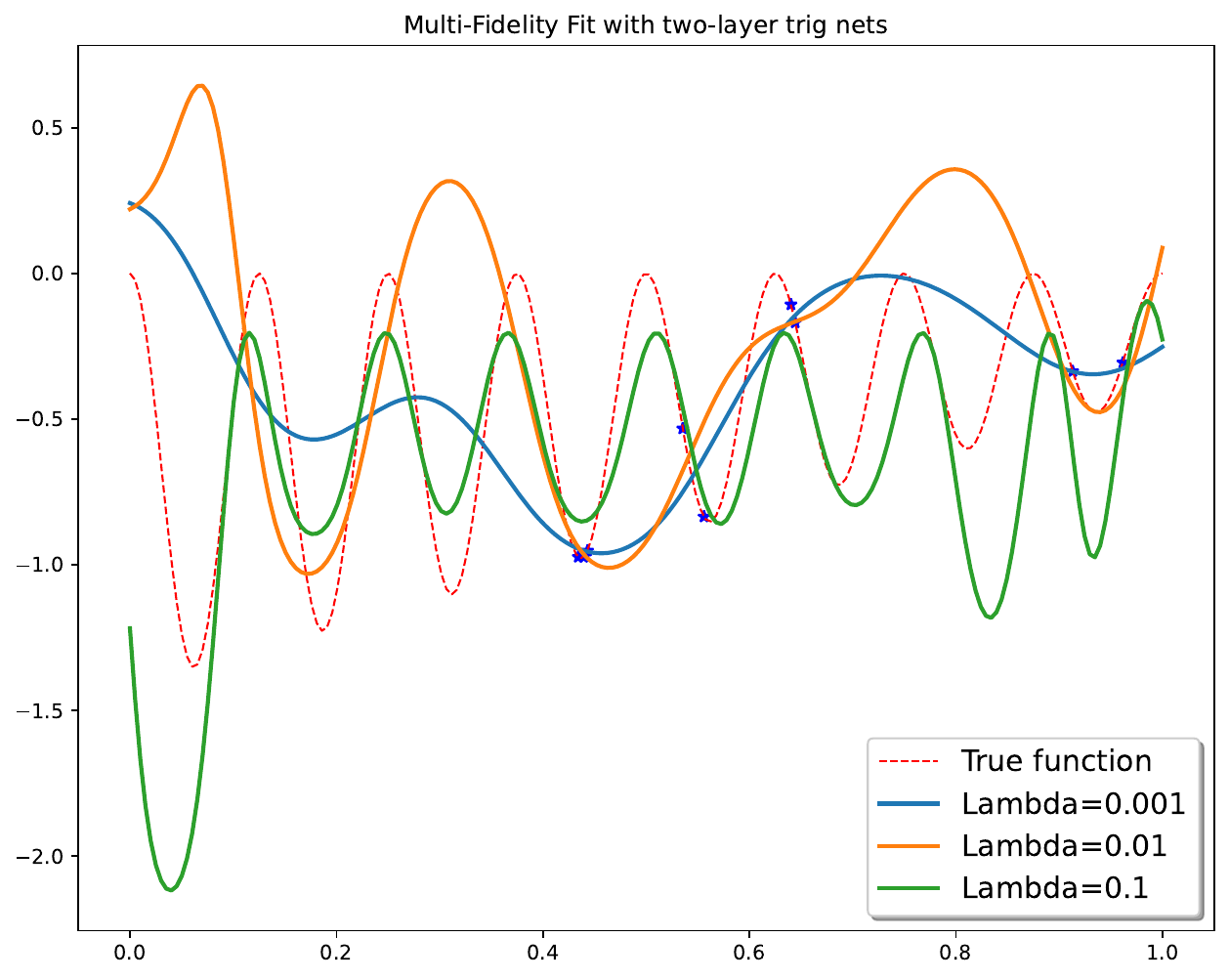}
%\fbox{\rule[-.5cm]{0cm}{4cm} \rule[-.5cm]{4cm}{0cm}}
\end{center}
\caption{Network model fitting the multi-fidelity data. The aim is to learn composite function $f(x)=h(g(x))$ with plenty of low fidelity data (not shown) seen from $g(x)$ and very rare data seen from $f(x)$ (blue dots generated from the red dashed ground truth). Left: GP fitting with the moment matching kernel. Right: deep trig net fitting with varying regularizing strength $\lambda_1$'s.}\label{multifidelity}
\end{figure}

In Bayesian learning, the structure of multi-fidelity DGP has the advantage of marginalizing the latent function $g$ conditioned on the low-fidelity data. As discussed in Sec.~\ref{sec:deep_mf}, the conditional mean and covariance for $g$ is translated into $\overline{\bf w}_1$ and precision matrix $A$ in weight space. Thus, the objective function for deep trig network learning becomes,
\begin{equation}
	\mathcal L = \sum_i[y_i-f({\bf x}_i)]^2 + \lambda_1({\bf w}_1-\overline{\bf w}_1)^tA{({\bf w}_1-\overline{\bf w}_1)}
	+\lambda_2{\bf w}^t_2{\bf w}_2\label{objective_2}\:,
	%[\log p({\bf w}_2)+\log p({\bf W}_1)]
\end{equation} where the two regularizing terms come from minus log of the prior over weights. In the right panel of Fig.~{\ref{multifidelity}}, one can see the predictive mean from using $\lambda_1=$ 0.001 (blue), 0.01 (orange), and 0.1 (green) given the high fidelity data (blue dots) generated from the true function (red dashed curve). As $\lambda_1$ increases, the knowledge, including uncertainty, about the latent function $g$ has more influence in learning the weight parameters through $\overline{\bf w}_1$ and $A$.

\subsection{Expressive shallow trig nets}

In the final subsection, we explore the possibility of enhancing the expressivity of shallow trig net by i). sampling the random frequencies from a mixture of Gaussians with nonzero centers, and ii) inserting a phase network before entering the sine/cosine activation units. With the shallow trig net, we can apply the standard linear Bayesian learning if the random frequencies in the feature function $\Phi(\Omega {\bf x})$ are fixed. We generate three different sets of frequencies from different mixtures of Gaussians $\sum_i\mathcal N(\mu_i\sigma_i^2)$. We use $(\mu_i,\sigma_i^2,\#)$ to denote the component center, variance, and number of samples. In Fig.~\ref{mixedGaussian} one can see the predictive mean (black dashed) sandwiched by $\pm2$ predictive std. The left panel is for a single Gaussian $\Omega\sim(0,25,75)$, middle for the mixture of $[(0,5,40), (50,25,35)]$, and the right for $[(0,5,25),(50,25,25),(100,25,25)]$. Given the same amount of activation units, the complexity of linear Bayesian model increases from the sampling $\Omega$ from a single zero-mean Gaussian to sampling from three Gaussians centered at 0, 50, and 100.

\begin{figure}[ht]
\begin{center}
%\framebox[4.0in]{$\;$}
\includegraphics[width=0.3\linewidth]{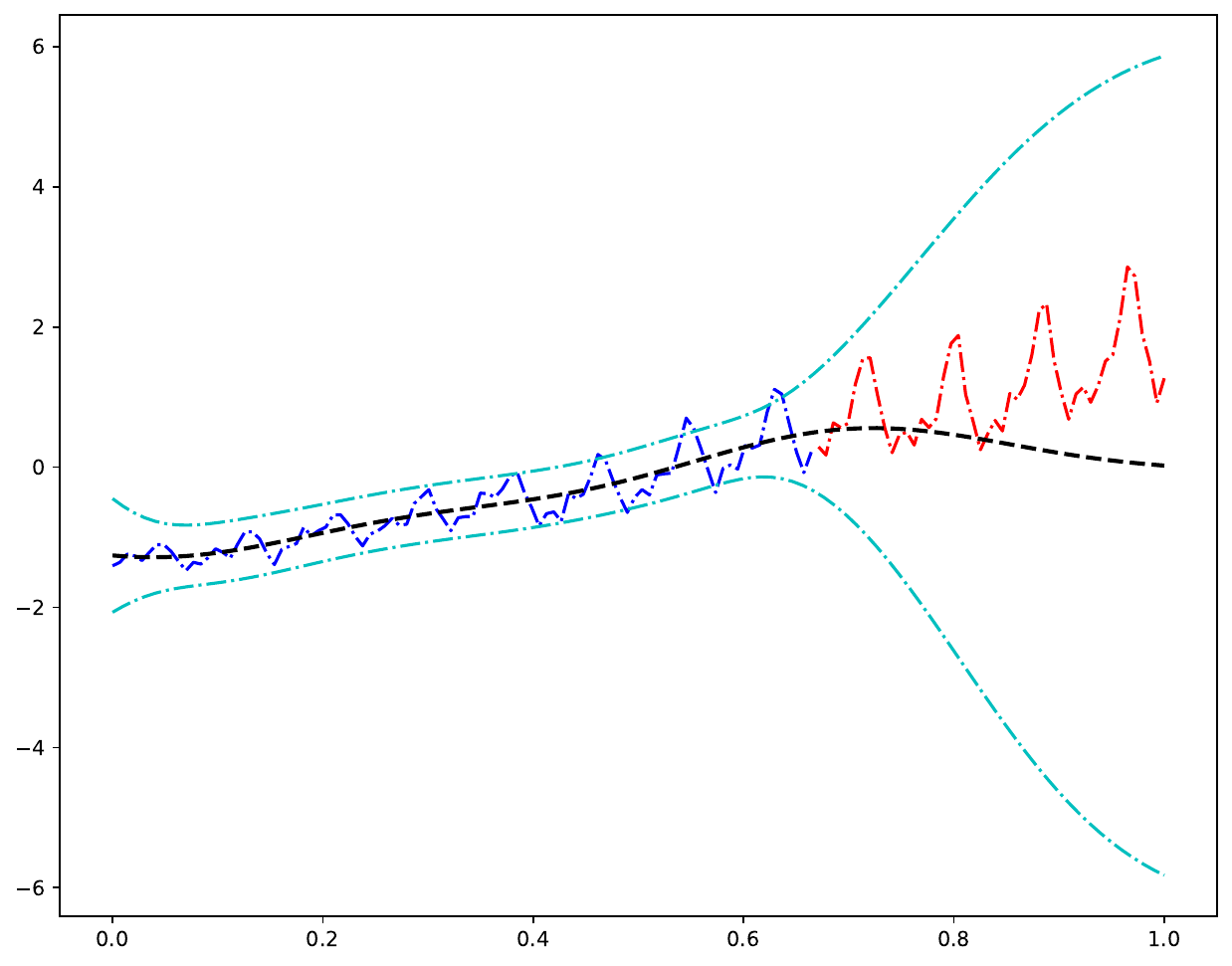}
\includegraphics[width=0.3\linewidth]{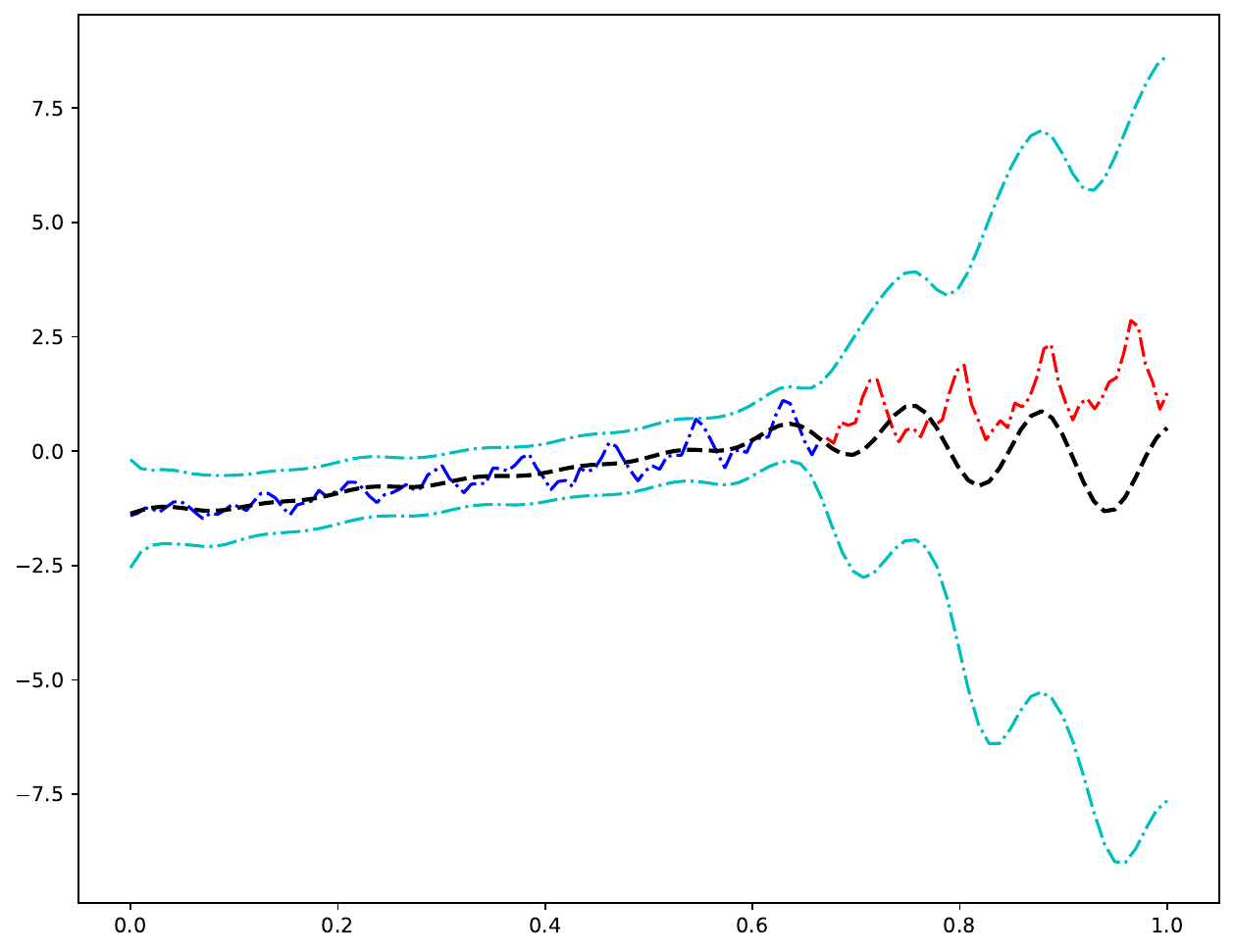}
\includegraphics[width=0.3\linewidth]{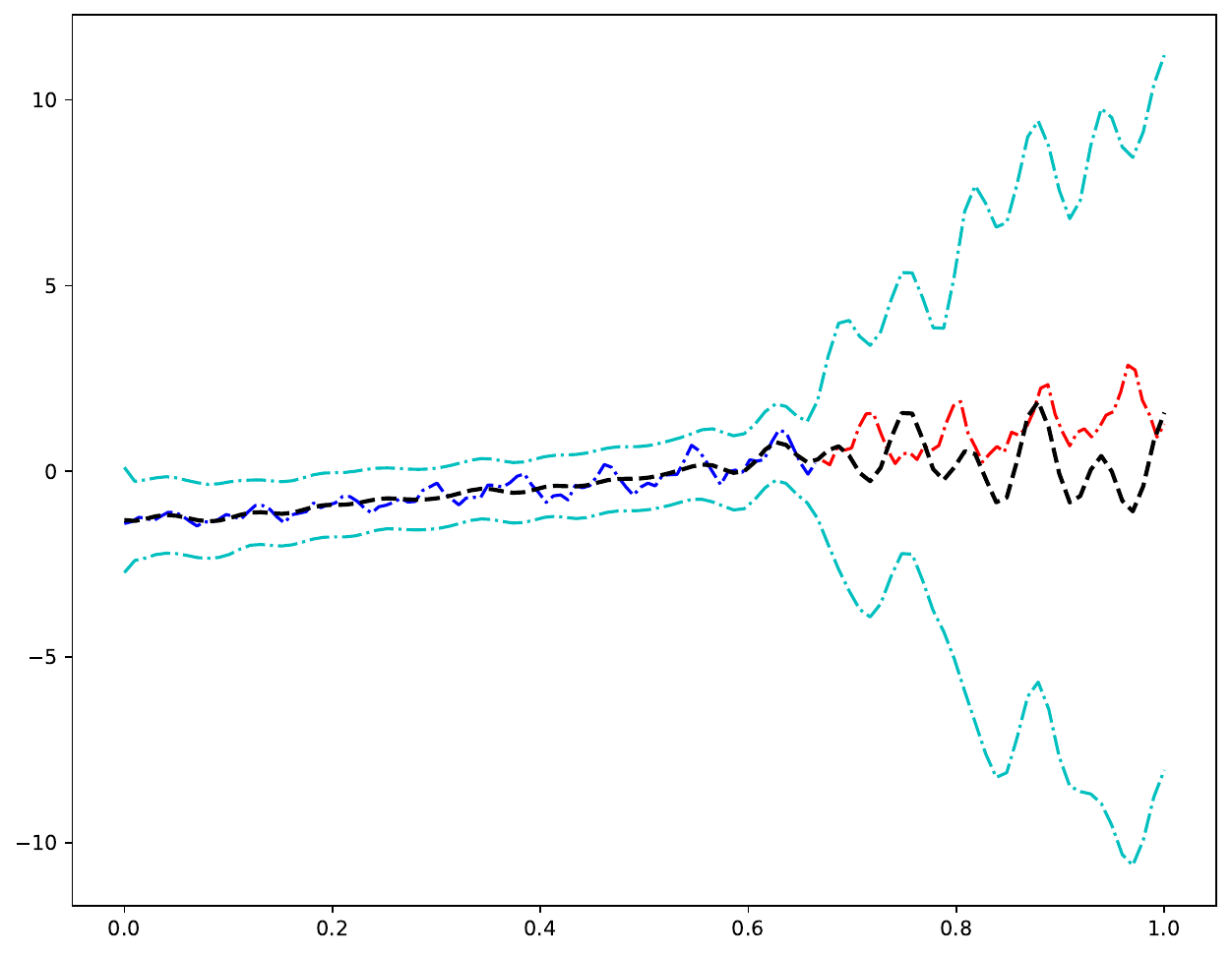}
%\fbox{\rule[-.5cm]{0cm}{4cm} \rule[-.5cm]{4cm}{0cm}}
\end{center}
\caption{Bayesian linear regression on airline passengers data set with variations in selecting the random frequencies $\Omega$ discussed in Sec.~\ref{sec:shallow}. See text for details of the mixture of Gaussians.}\label{mixedGaussian}
\end{figure}

Next, we are interested in fitting a pure noise data with the shallow trig net. As discussed in Sec.~\ref{subsec:phase}, the shallow network in Eq.~(\ref{single_network_phi}) with the inserted phase network $\Psi({\bf x})$ is shown to have non-Gaussian marginal prior. To see if the non-Gaussian character is related to its expressivity, we consider four different setups for fitting the noise (red points shown in Fig.~\ref{fit_noise}) generated from a normal distribution. In addition to the case without the phase network, a slight modification of Eq.~(\ref{single_network_phi}) in changing the sign of $\Psi$ within the sine function will lead the marginal prior distribution back to Gaussian. We implement $\Psi$ with another shallow width-50 ReLu network using PyTorch. In Fig.~\ref{fit_noise}, the predictive means from the vanilla GP and the shallow network without $\Psi$ are both linear with small slope, which is reasonable as the vanilla GP does not overfit. The phase network $\Psi$ does increase the expressivity of shallow network as the result (black solid) associated with Eq.~(\ref{single_network_phi}) is more influenced by the outliers than the modified one (with +/+ sign for $\Psi$) is. %The left panel shows the history of training error.

\begin{figure}[ht] 
\begin{center}
%\framebox[4.0in]{$\;$}
\includegraphics[width=0.4\linewidth]{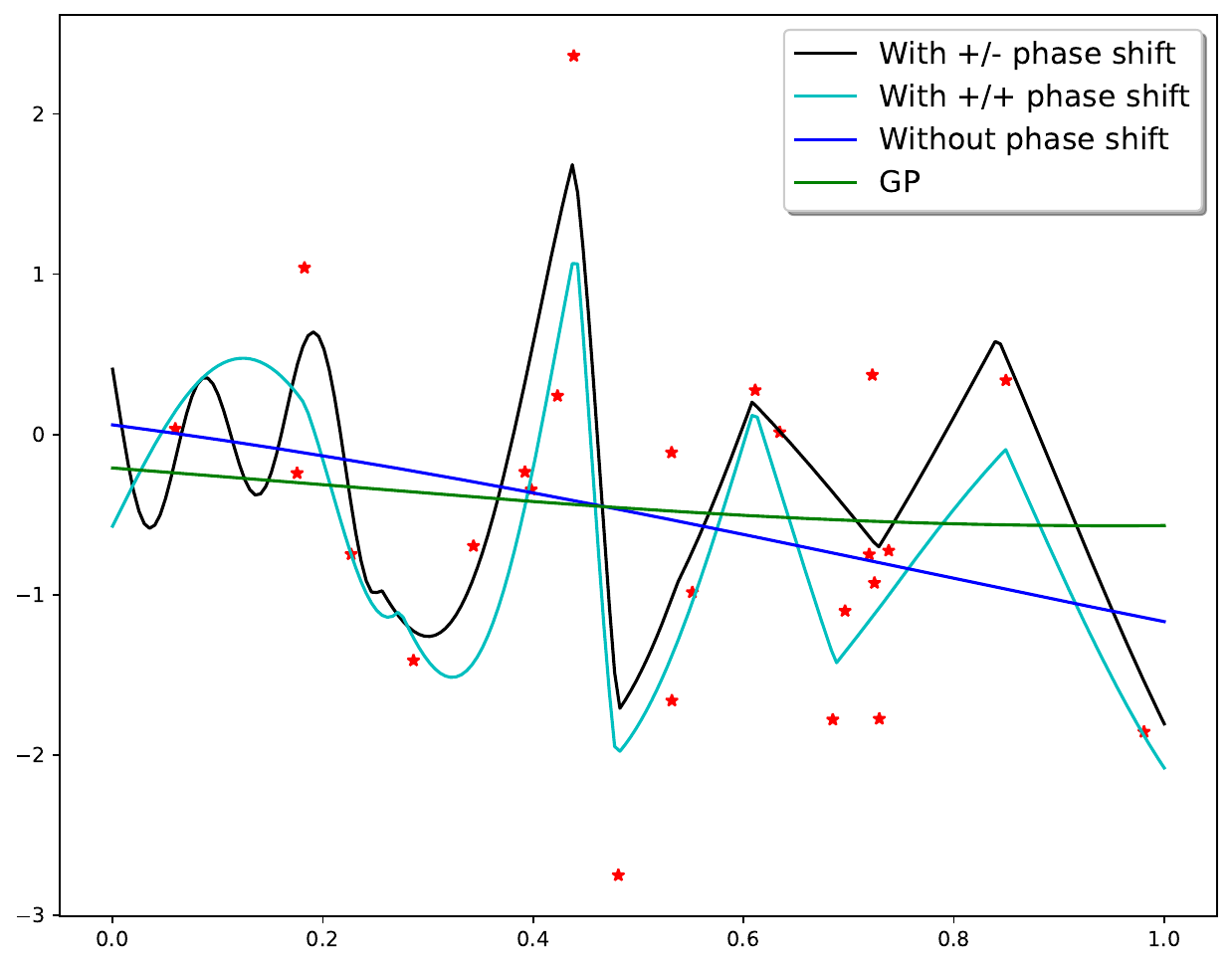}
%\includegraphics[width=0.36\linewidth]{Figure/Phase_loss.pdf}
%\fbox{\rule[-.5cm]{0cm}{4cm} \rule[-.5cm]{4cm}{0cm}}
\end{center}
\caption{Fitting noise data (red dots) with GP and 3 shallow networks.}
\label{fit_noise}
\end{figure}

\section{Related work}\label{sec:related}

%deep trigonometric networks were studied in \citep{sitzmann2020implicit} and due to the well behaved properties in terms of higher order derivatives the trig networks are able to model images and partial differential equations.

While \cite{neal1997monte} first pointed out the general correspondence between an infinite neural network and Gaussian process (GP), \cite{williams1997computing} demonstrated that neural networks with iid Gaussian weights and sigmoidal activation units are a representation of random functions drawn from a GP with arcsine covariance function. Later, \cite{cho2009kernel} obtained the arccosine kernel from computing the covariance of outputs from the ReLu neural networks. 
%With the closed form kernels, one is able to transform the prediction task based on a wide Bayesian neural network~\citep{mackay2003information} to a GP regression problem. 
%, avoiding the need of MCMC sampling, and have insights into the statistical behaviors~\citep{rasmussen2006gaussian}.
%able to model the predictive distribution using GPs, which otherwise can be done with     
%The implication of the analytic work is that it becomes possible to transform both regression and classification tasks 
Moreover, the correspondence holds beyond the shallow neural networks. \cite{matthews2018gaussian} and \cite{lee2018deep} studied the deep and wide neural networks and obtained a recursive relation for the emergent kernels. Similar techniques appeared in earlier work~\citep{schoenholz2016deep,Poole2016ExponentialEI} describing the statistics of forward and backward propagation with which phase transitions are identified in a number of learning phenomena. The connections between deep random networks and GPs were also studied extensively in~\citep{yang2019scaling}, and detailed effects of finite width can be found in~\citep{lee2020finite},
%from a smooth into a chaotic phase. %\cite{} \cite{li2021statistical}
%Functions propagating through the hidden layers start to show chaotic behaviors when the ratio of some hyper-parameters exceeds some limit, which quantitatively demonstrates the expressive power of deep nets~\citep{delalleau2011shallow,telgarsky2016benefits}. 
%obtained a closed form relation discerning hyper-parameter space 

%Deep Gaussian Process~\citep{damianou2013deep} (DGP) is a hierarchical composition of GPs. 
%Like deep networks are more expressive than the shallow counterparts, DGPs are expected to be more expressive than GPs. As there is no exact inference for DGP, its expressive power is usually demonstrated empirically on real-world datasets through approximate inference schemes~\citep{salimbeni2017doubly,salimbeni2019deep}. 
Theoretical progresses regarding understanding DGPs have been made via several important observations. In the deep limit, DGPs collapse to a constant function for some subspace of hyperparameters~\citep{duvenaud2014avoiding,dunlop2018deep,tong2021characterizing} and carry a heavy-tailed distribution over function derivatives~\citep{duvenaud2014avoiding}. \cite{lu2020interpretable} showed that the covariance and kurtosis are analytical characteristics of some two-layer DGPs, and a similar transition into chaotic phase with heavy-tailed multivariate statistics. Finite-width effects on statistics of the deep neural network were studied from field theory perspective~\citep{antognini2019finite,yaida2020non,roberts2021principles}, NTK perspective~\citep{hanin2019finite,arora2019exact}, and deep linear network~\citep{aitchison2020bigger}. 
%A DGP model is fully specified by the kernels and mean functions used in the hierarchical composition.

Deep bottlenecked network representation of DGP in weight space was first proposed by~\citep{cutajar2017random}, and ~\citep{mcdonald2021compositional} generalized the idea to include the latent force model for composing the kernels. \cite{agrawal2020wide} provided a formal and mathematical description for the connection. Uncertainty estimation in Bayesian deep neural network~\citep{wilson2020bayesian} can be done with variational inference~\citep{blundell2015weight}, ensemble method~\citep{lakshminarayanan2017simple}, random dropout~\citep{gal2016dropout}, and Laplace approximation~\citep{khan2019approximate}. The general issue about the underestimated in-between uncertainty due to the independent weight assumption in approximate posterior was addressed in~\citep{foong2020expressiveness}.

%, but the link between deep random networks and DGP was lacking. 
%In some deep models \pat{which ones??} , 
%In fact, deep neural networks can consist of some very wide hidden layers and some narrow ones. \cite{agrawal2020wide} demonstrated that such deep neural networks with bottlenecks are equivalent to DGPs. \cite{pleiss2021limitations} exploited the convexity of exponential function, showing bounded characteristic function of DGPs, and its convergence to that of a GP with some limiting kernel as the bottleneck width approaches infinity. Moreover, heavier statistical tails occur in the deeper models, and their conclusion is very general, not depending on the kernels of DGP or activation units of deep neural network. 
%Another line of thoughts are based on the random feature expansion of kernels such as the squared exponential (SE) one~\citep{rahimi2008random,gal2015improving},
%able to connect deep networks with DGPs originate from the idea of expanding the popular squared exponential kernels into an approximating sum of random features~\citep{rahimi2008random}, 
%The deep trigonometric network with bottleneck latent layer was proposed by \cite{cutajar2017random} as a proxy for the DGPs with SE kernels, based on the random Fourier expansion of SE kernels~\citep{rahimi2008random,gal2015improving}. Improvements in generalization over expectation propagation~\citep{bui2016deep} were demonstrated by employing a dedicated variational inference scheme~\citep{blundell2015weight,cutajar2017random}. 

\section{Conclusion}\label{sec:dis}
%\pat{it might make sense to more clearly indicate contributions here. starting from "we analytically ..." state the three contributions. then in a sentence or two each, state the significance of the contributions.}

%Induced by the network parameter prior distribution, w
More precise understanding of deep learning is critical for exploiting its expressive power and potential applications in high-stakes domains. 
In the wide limit as well as the case with finite width, we analytically investigated the covariance, marginal distribution, and neural tangent kernel of the trigonometric networks, connecting them with the deep Gaussian processes which can carry squared exponential kernel, spectral mixture kernel, and a combinations thereof.  
%Prior work demonstrates the effectiveness of the Bayesian deep trigonometric nets on real-world data \citep{cutajar2017random}. We analytically investigate the connection with deep Gaussian process. 
We have shown that deep Gaussian processes and deep trigonometric networks, one in function space and the other in weight space, yield the same covariance in a minimum model under various weight distributions. The derivation for the deep models in weight space is less intuitive, because it relies on an infinite dimensional Gaussian integral and knowledge of the spectrum of a particular random matrix. For deeper bottlenecked trig networks, the recursive relations [Eq.~(25) in~\citep{lu2020interpretable}] hold for the covariance approximately; without the bottlenecks the recursive relations [Eq.~(22) in ~\citep{duvenaud2014avoiding}] can describe the covariance.   
%\pat{I don't understand the next sentence. In what sense is it together? These feel like separate ideas.} Together with the heavy tail suggested by characteristic function associated with the multivariate density, the equivalence between the two deep models is established. These \pat{which? how?} insights also allow to explore the expressive power of deep trigonometric networks. \pat{I don't follow the next sentence} Feature parameters sampled from different distributions, correlated weight parameters, and the combined connect deep trigonometric networks with various other kernel regression methods. 
We have open a door to analyzing the effect of the non-Gaussianity of deep Gaussian process on its modeling power. Specifically, the derived neural tangent kernel $k_{\rm NTK}$ with deep trigonometric net representation allows the possibility of analyzing the implication of differences between $k_{\rm NTK}$ and the exact kernel $k_{\rm DGP}$ of deep Gaussian process, and the data-dependent kernels as a result of finite-width. 
%The large but finite deep trig networks seem to produce data-dependent kernels 
%and whether the predictive mean of intractable DGP inference can be obtained with $k_{\rm NTK}$.
%from the perspective of covariance and characteristic function.

%with squared exponential kernel with two layers of independent Gaussian weights.

%The flexibility and expressivity of deep trigonometriconometric networks through adopting different prior distribution over the network parameters are understood via investigating the analytical covariance and prior distribution over function output. We thus establish the equivalence with deep Gaussian process, allowing

% Acknowledgments---Will not appear in anonymized version
%\acks{We thank a bunch of people.}

\bibliography{msml22style/reference1}
\bibliographystyle{tmlr-style-file-main/tmlr}

%\appendix

%\section{My Proof of Theorem 1}

%This is a boring technical proof.

%\section{My Proof of Theorem 2}

%This is a complete version of a proof sketched in the main text.

\end{document}